\PassOptionsToPackage{breaklinks=true,draft=false,hypertexnames=false}{hyperref}
\PassOptionsToPackage{linesnumbered,noend}{algorithm2e}
\documentclass[final]{alt2026} %

\title[Optimal Agnostic Boosting with Improved Time]{Sample-Near-Optimal Agnostic Boosting with Improved Running Time}

\usepackage{mathtools,bbm,xspace}
\usepackage{bm}
\usepackage{color-edits}
\usepackage{float}
\usepackage{tikz}
\usetikzlibrary{decorations.pathreplacing, positioning}
\usetikzlibrary{calc}

\undef{\todo} %
\usepackage[
    textsize = tiny,
    colorinlistoftodos,
    obeyDraft,
]{todonotes}
\setuptodonotes{noinline}
\newcommand{\note}[2][]{}
\newcommand{\arthur}[2][]{}
\newcommand{\andrea}[2][]{}
\newcommand{\mikael}[2][]{}
\newcommand{\pagetodo}[2][]{}

\usepackage{thmtools}
\usepackage{thm-restate}

\usepackage{algorithm,algpseudocode}
\usepackage{amsmath}
\usepackage{amsfonts}
\usepackage{amssymb}

\DeclareMathOperator{\ls}{\mathcal{L}}

\DeclareMathOperator*{\p}{\mathbb{P}}

\DeclareMathOperator{\cor}{\mathit{cor}}
\DeclareMathOperator{\fs}{\mathnormal{f^{\ast}}}
\DeclareMathOperator{\Ln}{\mathrm{Ln}} %

\SetCommentSty{mycommentfont}
\usepackage{etoolbox}
\makeatletter
\patchcmd{\@algocf@start}%
{-1.5em}%
{-2pt}%
{}{}%
\makeatother

\usepackage{etoolbox}
\usepackage{afterpackage}
\usepackage{amsmath}
\usepackage{amssymb}
\usepackage{stmaryrd}
\usepackage[]{mathtools}
\usepackage{bm}

\usepackage{fancyref}
\usepackage{autonum}

\newcommand{\expanderFref}[1]{\expandafter\Fref\expandafter{#1}}
\let\cref\expanderFref
\let\Cref\cref

\newcommand*{\fancyrefalglabelprefix}{alg}      
\newcommand*{\fancyreflinelabelprefix}{line}    
\newcommand*{\fancyrefthmlabelprefix}{thm}      
\newcommand*{\fancyreflemlabelprefix}{lem}      
\newcommand*{\fancyrefcorlabelprefix}{cor}      
\newcommand*{\fancyrefproplabelprefix}{prop}    
\newcommand*{\fancyrefdeflabelprefix}{def}      
\newcommand*{\fancyrefremlabelprefix}{rem}      
\newcommand*{\fancyrefobslabelprefix}{obs}      
\newcommand*{\fancyrefapplabelprefix}{app}      
\newcommand*{\fancyreftemplabelprefix}{temp}    

\frefformat{plain}{\fancyrefalglabelprefix}{\methodname{algorithm\fancyrefdefaultspacing#1}}
\Frefformat{plain}{\fancyrefalglabelprefix}{\methodname{Algorithm\fancyrefdefaultspacing#1}}
\frefformat{plain}{\fancyreflinelabelprefix}{line\fancyrefdefaultspacing#1}
\Frefformat{plain}{\fancyreflinelabelprefix}{Line\fancyrefdefaultspacing#1}
\frefformat{plain}{\fancyrefthmlabelprefix}{theorem\fancyrefdefaultspacing#1}
\Frefformat{plain}{\fancyrefthmlabelprefix}{Theorem\fancyrefdefaultspacing#1}
\frefformat{plain}{\fancyreflemlabelprefix}{lemma\fancyrefdefaultspacing#1}
\Frefformat{plain}{\fancyreflemlabelprefix}{Lemma\fancyrefdefaultspacing#1}
\frefformat{plain}{\fancyrefcorlabelprefix}{corollary\fancyrefdefaultspacing#1}
\Frefformat{plain}{\fancyrefcorlabelprefix}{Corollary\fancyrefdefaultspacing#1}
\frefformat{plain}{\fancyrefproplabelprefix}{proposition\fancyrefdefaultspacing#1}
\Frefformat{plain}{\fancyrefproplabelprefix}{Proposition\fancyrefdefaultspacing#1}
\frefformat{plain}{\fancyrefdeflabelprefix}{definition\fancyrefdefaultspacing#1}
\Frefformat{plain}{\fancyrefdeflabelprefix}{Definition\fancyrefdefaultspacing#1}
\frefformat{plain}{\fancyrefremlabelprefix}{remark\fancyrefdefaultspacing#1}
\Frefformat{plain}{\fancyrefremlabelprefix}{Remark\fancyrefdefaultspacing#1}
\frefformat{plain}{\fancyrefobslabelprefix}{observation\fancyrefdefaultspacing#1}
\Frefformat{plain}{\fancyrefobslabelprefix}{Observation\fancyrefdefaultspacing#1}
\frefformat{plain}{\fancyrefapplabelprefix}{appendix\fancyrefdefaultspacing#1}
\Frefformat{plain}{\fancyrefapplabelprefix}{Appendix\fancyrefdefaultspacing#1}
\frefformat{plain}{\fancyreftemplabelprefix}{TELL ARTHUR ABOUT THIS\fancyrefdefaultspacing#1}
\Frefformat{plain}{\fancyreftemplabelprefix}{TELL ARTHUR ABOUT THIS\fancyrefdefaultspacing#1}

\makeatletter

\newcommand{\AfterThisEnv}[1]{%
  \AddToHookNext{env/\@currenvir/after}{#1}%
}
\makeatother

\newcounter{templabel}

\newcommand{\llasttemplabel}{eq:temp:\the\numexpr\value{templabel} - 1\relax}
\newcommand{\lllasttemplabel}{eq:temp:\the\numexpr\value{templabel} - 2\relax}
\newcommand{\llllasttemplabel}{eq:temp:\the\numexpr\value{templabel} - 3\relax}

\newcommand*{\bbig}[1]{\Big{#1}}

\newcommand*{\bbbig}[1]{\bigg{#1}}

\newcommand*{\bbbbig}[1]{\Bigg{#1}}
\newcommand*{\bbbbigl}[1]{\Biggl{#1}}
\newcommand*{\bbbbigr}[1]{\Biggr{#1}}

\newcommand{\placeholderarg}{{\mathchoice{\:}{\:}{\,}{\,}\cdot\mathchoice{\:}{\:}{\,}{\,}}}

\DeclarePairedDelimiter\Br() 
\DeclarePairedDelimiter\Ceil\lceil\rceil

\newcommand{\multibinom}[2]{%
\left.\mathchoice
  {\left(\kern-0.48em\binom{#1}{#2}\kern-0.48em\right)}
  {\big(\kern-0.30em\binom{\smash{#1}}{\smash{#2}}\kern-0.30em\big)}
  {\left(\kern-0.30em\binom{\smash{#1}}{\smash{#2}}\kern-0.30em\right)}
  {\left(\kern-0.30em\binom{\smash{#1}}{\smash{#2}}\kern-0.30em\right)}
\right.}

\providecommand\given{\errmessage{Used `given` command outside proper context}}
\newcommand\GivenSymbol[1][]{%
    \mathchoice{\:}{\:}{\,}{\,}#1\vert%
    \allowbreak%
    \mathchoice{\:}{\:}{\,}{\,}%
    \mathopen{}%
}

\DeclarePairedDelimiterX\Set[1]\{\}{%
    \renewcommand\given{\GivenSymbol[\delimsize]}%
    #1%
}
\DeclarePairedDelimiterXPP\BrWithGiven[1]{}{(}{)}{}{%
    \renewcommand\given{\GivenSymbol[\delimsize]}%
    #1%
}
\DeclarePairedDelimiterXPP\BrWithGGiven[1]{}{(}{)}{}{%
    \renewcommand\given{\GGivenSymbol[\delimsize]}%
    #1%
}
\DeclarePairedDelimiterXPP\SBrWithGiven[1]{}{[}{]}{}{%
    \renewcommand\given{\GivenSymbol[\delimsize]}%
    #1%
}

\NewDocumentCommand\Prob{ e{_^} }{
    \operatorname*{\mathbb{P}}
    \IfNoValueF{#1}{\sb{#1}}
    \IfNoValueF{#2}{\errmessage{Not supposed to have superscript}\sp{#2}}
    \SBrWithGiven
}
\NewDocumentCommand\Ev{ e{_^} }{
    \operatorname*{\mathbb{E}}
    \IfNoValueF{#1}{\sb{#1}}
    \IfNoValueF{#2}{\errmessage{Not supposed to have superscript}\sp{#2}}
    \SBrWithGiven
}
\NewDocumentCommand\Var{ e{_^} }{
    \operatorname{Var}
    \IfNoValueF{#1}{\sb{#1}}
    \IfNoValueF{#2}{\errmessage{Not supposed to have superscript}\sp{#2}}
    \BrWithGiven
}
\NewDocumentCommand\Cov{ e{_^} }{
    \operatorname{Cov}
    \IfNoValueF{#1}{\sb{#1}}
    \IfNoValueF{#2}{\errmessage{Not supposed to have superscript}\sp{#2}}
    \BrWithGiven
}
\NewDocumentCommand\KL{ e{_^} }{
    \operatorname{KL}
    \IfNoValueF{#1}{\sb{#1}}
    \IfNoValueF{#2}{\errmessage{Not supposed to have superscript}\sp{#2}}
    \BrWithGGiven
}

\NewDocumentCommand{\Fr}{ ssmm }{%
    \IfBooleanTF{#1}{%
        \IfBooleanTF{#2}{%
            #3/(#4)%
        }{%
            #3/#4%
        }%
    }{%
        \mathchoice{%
            \dfrac{#3}{#4}
        }{%
            \sfrac{#3}{#4}
        }{%
            \sfrac{#3}{#4}
        }{%
            \sfrac{#3}{#4}
        }%
    }%
}

\let\origfrac\frac
\NewDocumentCommand{\flatfrac}{ O{}O{}mm }{#1{#3}/#2{#4}}
\NewDocumentCommand{\FFr}{ sO{}O{}mm }{%
    \IfBooleanTF{#1}{%
        \origfrac{#2{#4}}{#3{#5}}
    }{%
        \mathchoice{%
            \origfrac{#4}{#5}
        }{%
            \flatfrac[#2][#3]{#4}{#5}
        }{%
            \flatfrac[#2][#3]{#4}{#5}
        }{%
            \flatfrac[#2][#3]{#4}{#5}
        }%
    }%
}
\let\Frac\FFr

\newcommand*{\methodname}[1]{\ensuremath{\operatorname{\mathsf{#1}}}}

\DeclareMathOperator{\sign}{sign}

\DeclareMathOperator{\DistribOver}{\Delta}
\DeclareMathOperator{\Uniform}{Uniform}

\DeclarePairedDelimiter\abs{\lvert}{\rvert}

\DeclarePairedDelimiterXPP\normp[1]{}\lVert\rVert{_p}{#1}
\DeclarePairedDelimiterXPP\normzero[1]{}\lVert\rVert{_0}{#1}
\DeclarePairedDelimiterXPP\normone[1]{}\lVert\rVert{_1}{#1}
\DeclarePairedDelimiterXPP\normtwo[1]{}\lVert\rVert{_2}{#1}
\DeclarePairedDelimiterXPP\norminf[1]{}\lVert\rVert{_\infty}{#1}
\DeclarePairedDelimiterXPP\normmax[1]{}\lVert\rVert{_\mathrm{max}}{#1}
\DeclarePairedDelimiterXPP\normspec[1]{}\lVert\rVert{_\mathrm{spectral}}{#1}
\let\normmax\norminf

\NewDocumentCommand\BallClosed{ e{_^} }{
    \operatorname{\mathfrak{B}}_]
    \IfNoValueF{#1}{\errmessage{Not supposed to have subscript}\sb{#1}}
    \IfNoValueF{#2}{\sp{#2}}
    \Br
}
\NewDocumentCommand\BallOpen{ e{_^} }{
    \operatorname{\mathfrak{B}}_)
    \IfNoValueF{#1}{\errmessage{Not supposed to have subscript}\sb{#1}}
    \IfNoValueF{#2}{\sp{#2}}
    \Br
}

\let\eps\varepsilon

\renewcommand*{\restriction}[1]{\mathord{\upharpoonright}_{#1}} 
\renewcommand*{\restriction}[1]{\mathord{|}_{#1}} 

\newcommand*{\indicator}[1]{\bm{1}_{\Set{#1}}}    

\DeclareMathOperator*{\argmin}{arg\,min\,}

\DeclareMathOperator{\Loss}{\mathcal{L}}
\DeclareMathOperator{\err}{err}
\newcommand*{\errst}{\mathord{\mathrm{err}^{*}}}
\DeclareMathOperator{\corr}{cor}
\let\cor\corr
\DeclareMathOperator{\Weaklearner}{\mathcal{W}}
\newcommand*{\InputSpace}{\mathcal{X}}

\DeclareMathOperator{\VCdim}{VC}

\input{mathstyle}

\usepackage{times}

\altauthor{\Name{Arthur {da Cunha}} \Email{dac@cs.au.dk}\\
  \Name{Mikael {M{\o}ller H{\o}gsgaard}} \Email{hogsgaard@cs.au.dk}\\
  \Name{Andrea Paudice} \Email{apaudice@cs.au.dk}\\
  \addr Aarhus University%
}

\begin{document}

\maketitle

\begin{abstract}%
  Boosting is a powerful method that turns weak learners, which perform only slightly better than random guessing, into strong learners with high accuracy.
While boosting is well understood in the classic setting, it is less so in the agnostic case, where no assumptions are made about the data.
Indeed, only recently was the sample complexity of agnostic boosting nearly settled \citep{Us25}, but the known algorithm achieving this bound has exponential running time.
In this work, we propose the first agnostic boosting algorithm with near-optimal sample complexity, running in time polynomial in the sample size when considering the other parameters of the problem fixed.
\end{abstract}

\begin{keywords}%
  Boosting, agnostic learning, sample complexity.
\end{keywords}

\section{Introduction}
\label{sec:intro}

Boosting \citep{Schapire90} is a highly successful learning technique that, in the realizable setting, can produce models with arbitrarily good accuracy starting from models that perform only slightly better than random guessing.
More precisely, boosting establishes the computational equivalence between two learning models that at first might seem fundamentally distant:
\emph{(strong) Probably Approximately Correct (PAC) learning} \citep{Valiant84},
and \emph{weak (PAC) learning} \citep{KearnsV89}.
To make these concepts formal,
consider an input space $\InputSpace$,
a hypothesis class $\fF \subseteq \Set{\pm 1}^\InputSpace$.
A distribution $\dD$ over $\InputSpace \times \Set{\pm 1}$ is called \emph{realizable (relative to $\fF$)} if it satisfies that some target hypothesis $f \in \fF$ achieves zero error under $\dD$,
that is,
\begin{align}
    \exists f \in \fF
    \enspace\text{such that}\enspace
    \err_{\dD}(f)
    \coloneqq \Prob_{(\rx, \ry) \sim \dD}{f(\rx) \ne \ry}
    = 0
    .
    \label{eq:realizability}
\end{align}

A \emph{(strong) PAC learner} is an algorithm that,
for any desired accuracy and confidence parameters, $\eps, \delta \in (0, 1)$,
and any realizable distribution $\dD$ relative to $\fF$,
can learn from a training sequence $\rsS \sim \dD^n$ of size $n = n(\eps, \delta)$ and,
with probability at least $1 - \delta$ over the drawing of $\rsS$,
find a classifier $g\colon \InputSpace \to \Set{\pm 1}$ with
\begin{align}
    \err_{\dD}(g)
    \le \eps
    .
\end{align}
The amount of training data $n(\eps, \delta)$ necessary for the algorithm to reach this goal is called its \emph{sample complexity}.

On the other hand,
given $\gamma \in (0, 1/2]$,
a \emph{$\gamma$-weak learner for $\fF$ with base class $\fH \subseteq \Set{\pm 1}^\InputSpace$} is an algorithm that,
given $m_{0}$ i.i.d.\ samples from an arbitrary realizable distribution $\dD'$ (relative to $\fF$),
returns a classifier $w \in \fH$ such that,
with probability at least $1 - \delta_{0}$ over the drawing of the input samples,
\begin{align}
    \err_{\dD'}(w)
    \le 1/2 - \gamma
    .
\end{align}
That is, for small values of $\gamma$, the algorithm performs only slightly better than random guessing, which would achieve an error of $\Frac{1}{2}$.
Accordingly, the base class $\fH$ is typically much simpler than $\fF$.

Naturally,
the existence of a strong learner for a given task implies the existence of a weak learner for the same task.
Strikingly,
\citet{Schapire90} established the reciprocal: weak learnability implies strong learnability.
The equivalence is realized by \emph{boosting algorithms},
such as the seminal \methodname{AdaBoost} \citep{FreundS97},
which construct a strong PAC learner by repeatedly invoking a weak learner on adaptively reweighted distributions over the training sequence,
and aggregating the returned classifiers into a final hypothesis.
Crucially, the number of calls to the weak learner depends only logarithmically on the sample size.
Namely, $O\Br*{\log(n)/\gamma^2}$ calls suffice for \methodname{AdaBoost} to be a strong PAC learner, with $n_{\methodname{AdaBoost}}(\eps, \delta) = \tilde{O}\Br*{\VCdim(\fH)/(\eps\gamma^2) + \log(1/\delta)/\eps}$ \citep{FreundS97},
where $\tilde{O}(\placeholderarg)$ and analogous notations hide logarithmic factors.
Moreover, the sample complexity of boosting is well settled in the realizable setting to be of order $\Theta\Br*{\VCdim(\fH)/(\eps\gamma^2) + \log(1/\delta)/\eps}$ \citep{LarsenR22}.
This has made realizable boosting highly successful both in theory and in practice,
as witnessed by the rich literature in both directions.
We refer the reader to \citet{boostingbook} for a comprehensive survey.

Despite their success,
\methodname{AdaBoost} and its usual variants turn out to rely heavily on the realizability of the data distribution \citep[Section~12.3]{boostingbook}.
This is a significant shortcoming since real data rarely satisfies \cref{eq:realizability}, due to noise in the labels, or the sheer fact that the hypothesis class considered might lack the capacity to represent the target function, for example.
To accommodate this,
the \emph{agnostic PAC learning} model \citep{Haussler92,KearnsSS94} forgoes \emph{all} assumptions on the data distribution.
Of course, doing so implies that arbitrarily good accuracy may no longer be possible.
Accordingly, an \emph{agnostic PAC learner} is required to find a classifier with arbitrarily small excess error relative to the best performing hypothesis in $\fF$.
More precisely,
for any $\eps, \delta \in (0, 1)$ and any distribution $\dD$ over $\InputSpace \times \Set{\pm 1}$,
with probability at least $1 - \delta$ over the drawing of $\rsS \sim \dD^{n}$, with $n = n(\eps, \delta)$,
the algorithm returns $g\colon \InputSpace \to \Set{\pm 1}$ such that
\begin{align}
    \err_{\dD}(g)
    \le \inf_{f \in \fF} \err_{\dD}(f) + \eps
    .
\end{align}
When $\fF$ and $\dD$ are clear from the context, we let $\errst \coloneqq \inf_{f \in \fF} \err_{\dD}(f)$.

Compared with the realizable counterpart,
the study of boosting in the agnostic setting is more recent,
starting with \citet{Ben-DavidLM01} and including several contributions proposing efficient algorithms for it \citep{KalaiK09,BrukhimCHM20,GhaiS24,GhaiS25}.
Notably,
the sample complexity of agnostic boosting was not well understood until the very recent work of \citet{Us25},
which established a lower bound on the sample complexity of any agnostic boosting algorithm, and proposed an algorithm matching it up to logarithmic factors.
These results are associated with a definition of agnostic weak learner.
The work of \citet{Us25} employs a particularly general one, originally from \citet{GhaiS24}, which we also adopt.{\footnotemark}
\footnotetext{Agnostic boosting lacks uniformity in definitions compared to the realizable case.
We refer the reader to \citep[Appendix~A]{Us25} for a discussion of the nuances of this matter.}

\begin{definition}[Agnostic Weak Learner]\label{def:agnosticweaklearner}
  Let $\gamma_{0}, \eps_{0}, \delta_{0} \in [0, 1]$,
  $m_{0} \in \N$,
  $\fF \subseteq \Set{\pm 1}^{\InputSpace}$,
  and $\fH \subseteq \Set{\pm 1}^{\InputSpace}$.
  A (randomized){\footnotemark} learning algorithm $\Weaklearner\colon (\InputSpace \times \Set{\pm 1})^* \to \fH$ is called a \emph{$(\gamma_{0}, \eps_{0}, \delta_{0}, m_{0})$ agnostic weak learner} for reference class $\fF$ with base class $\fH$
  iff:
  for any distribution $\fD$ over $\InputSpace \times \Set{\pm 1}$,
  given sample $\rsS \sim \fD^{m_{0}}$,
  with probability at least $1 - \delta_{0}$
  over the drawing of $\rsS$
  the hypothesis $\rw = \Weaklearner(\rsS)$ satisfies that
  \begin{align}
    \corr_{\fD}(\rw)
    \coloneqq \Ev_{(\rx, \ry) \sim \fD}{\ry \cdot \rw(\rx)}
    \ge \gamma_{0} \cdot \sup_{f \in \fF} \corr_{\fD}(f) - \eps_{0}
    .
    \label{eq:agnosticweaklearner}
  \end{align}
\end{definition}
\footnotetext{To accommodate for common practices such as the use of random feature selection,
we allow for a randomized weak learner that takes an additional random seed (independent of the sample) as input.
More formally, given a distribution $\dU$ over a set of random seeds, and we consider $\Weaklearner$ such that
\bgroup
\setlength\belowdisplayskip{0pt}
\begin{align}
    \Prob_{\rsS \sim \dD^{m_{0}}, \rb \sim \rU}{\corr_{\dD}(\Weaklearner(\rsS, \rb)) \ge \gamma_{0} \cdot \sup_{f \in \fF} \corr_{\dD}(f) - \eps_{0}}
    \ge 1 - \delta_{0}
    .
\end{align}
\egroup}

The definition above employs the \emph{correlation} of a hypothesis under a distribution $\dD$.
This is meant to ease comparison with previous work, and one can readily translate it to use $\err$ since for any binary-valued hypothesis $h\colon \InputSpace \to \Set{\pm 1}$,
\begin{align}
    \corr_{\dD}(h)
    = 1 - 2 \err_{\dD}(h)
    .
\end{align}
In particular, random guessing yields zero correlation in expectation, while perfect classification achieves correlation of one.

Under \cref{def:agnosticweaklearner},
\citet{Us25} show a hard instance
consisting of
a domain $\InputSpace$,
and hypothesis classes $\fF, \fH \subseteq \Set{\pm 1}^{\InputSpace}$,
for which there exists a $(\gamma_{0}, \eps_{0}, \delta_{0}, m_{0})$ agnostic weak learner for $\fF$ with base class $\fH$,
and, \emph{nonetheless},
ensuring excess error of at most $\eps$ with constant probability requires a sample size of at least
\begin{align}
    \tilde{\Omega}\Br*{\Frac{\VCdim(\fH)}{\gamma_{0}^{2} \eps^{2}}}.
    \label{eq:lb-sample-complexity}
\end{align}
For convenience, we include the full statement of this bound in \cref{app:lowerbound} (\cref{thm:lowerbound}).

Moreover, the authors prove that their proposed algorithm achieves a sample complexity matching \cref{eq:lb-sample-complexity} up to logarithmic factors.
Thus, they nearly settle the sample complexity of agnostic boosting.
However, despite its near-optimal sample complexity,
the algorithm proposed by \citet{Us25} is inherently inefficient in terms of computational effort:
it calls the weak learner (which usually dominates the overall running time) an exponential number of times.
This is in contrast with other works \citep{GhaiS24,GhaiS25,BrukhimCHM20,KalaiK09}, which propose efficient methods, albeit all with at least a $1/(\gamma_{0}\eps)$ overhead in their sample complexity compared to \cref{eq:lb-sample-complexity}.

In this work, we bridge this gap by proposing the first agnostic boosting algorithm (\cref{alg:main}) with near-optimal sample complexity and polynomial running time when the other parameters of the problem is considered fixed.
Namely, it achieves the guarantee in \cref{eq:lb-sample-complexity} while maintaining the running time polynomial in the sample size.
Precisely, we prove the following.
\vspace{0.9cm}
\begin{theorem}[Asymptotic version of \protect\cref{thm:main-full}]\label{thm:main}
    Given $\gamma_{0}, \eps_{0}, \delta_{0} \in (0, 1)$,
    $m_{0} \in \N$,
    and $\fH, \fF \subseteq \Set{\pm 1}^{\InputSpace}$,
    let $\theta = (\gamma_{0} - \eps_{0})/2$,
    let $d$ be the VC dimension of $\fH$ and $d^{*}$ its dual\footnote{See \cref{def:dual-vc}.} VC dimension,
    and let $\Weaklearner$ be a $(\gamma_{0}, \eps_{0}, \delta_{0}, m_{0})$ agnostic weak learner for $\fF$ with base class $\fH$.
    Then,
    for any $\delta \in (0, 1)$,
    $n \in \N$,
    and distribution $\dD$ over $\InputSpace \times \Set{\pm 1}$,
    given $\rsS \sim \dD^{n}$,
    \Cref{alg:main}, when executed on input $(\rsS, \delta, \Weaklearner, \delta_{0}, m_{0}, \theta, d^*)$,
    returns,
    with probability at least $1 - \delta$ over $\rsS$ and the randomness of $\Weaklearner$,
    a classifier $v\colon \InputSpace \to \Set{\pm 1}$ such that
    \begin{align}
        \err_{\dD}(v)
        &\le \errst + O\Br[\bbbbig]{
            \sqrt{\errst \cdot \Frac{d'\Ln\Br{\Frac*{n}{d'}} + \ln\Br{\Frac*{1}{\delta}}}{n}}
            + \Frac{d'\Ln\Br{\Frac*{n}{d'}} + \ln\Br{\Frac*{1}{\delta}}}{n}
        }
        ,
        \label{eq:main-err-bound}
    \end{align}
    where $\errst = \inf_{f \in \fF} \err_{\dD}(f)$
    and $d' = O\Br[\big]{Td\ln\Br{T}}$,
    with $T = O\Br[\big]{\Ceil{\min\Set{\Frac{\ln\Br{n}}{\theta^{2}},\, \Frac[]{d^{*}}{\theta^{2}}}}}$.

    Moreover,
    whenever the above bound is not vacuous,
    i.e., for $n = \Omega(\max\Set{dT,\, \ln\Br{1/\delta}})$,
    \cref{alg:main} invokes $\Weaklearner$ at most
    \begin{math}
        O\Br{n^{m_{0} + 3}}
    \end{math}
    times,
    and the running time of \cref{alg:main} is at most
    \begin{align}
        \operatorname{Eval}_{\fH}(1) \cdot n^{O\Br*{m_{0} \cdot \min\{d^{*},\, \ln\Br{n}\}/\theta^2}}
        ,
    \end{align}
    where $\operatorname{Eval}_{\fH}(1)$ is the time it takes to evaluate a single hypothesis from $\fH$ on a single point.
\end{theorem}

Solving for $n$ in \cref{eq:main-err-bound}, and, for simplicity, omitting the $\errst$ term, yields a sample complexity of
\begin{align}
    \tilde{O}\Br*{\Frac{\VCdim(\fH)}{\theta^2 \eps^{2}}},
\end{align}
thus, matching the near-optimality of the algorithm by \citet{Us25}, and matching the lower bound in \cref{eq:lb-sample-complexity} up to logarithmic factors whenever $\theta = \Omega(\gamma_{0})$.
Taking $\errst = \inf_{f \in \fF} \err_{\dD}(f)$ into account reveals a better sample complexity whenever this infimum is small.
Specifically, if $\errst = \tilde{O}(\eps)$, then the sample complexity is of order $\tilde{O}(d/(\theta^2\eps))$, matching the optimal sample complexity for the realizable case up to logarithmic factors.
This ensures that \cref{alg:main} achieves better generalization when the underlying distribution is closer to being realized by the reference class $\fF$.
In fact, the exact lower bound by \citet{Us25} (\cref{thm:lowerbound}) implies that this dependence on the error is optimal up to logarithmic factors.

The core point of \cref{thm:main} is that, while the previous nearly statistically (sample) optimal method, by \citet{Us25}, invokes the weak learner an exponential (in the sample size $n$) number of times, \cref{alg:main} invokes $\Weaklearner$ only a polynomial number of times, leading to a polynomial overall running time.
The parameters in the exponent only depend on the properties of the agnostic weak learner, which may not be so adversarial in many settings of interest.
In particular, the dependence on $d^{*}$ may seem concerning as, in the worst case, it can depend exponentially on $d$ \citep{Assouad83}.
Yet, boosting usually employs simple base classes $\fH$, and for many of those, such as geometrically defined classes in $\R^r$, we have $d^{*}, d \le r$ \citep[paragraph following Claim~1.6]{MoranY16}.
Moreover, for the natural and common choice of halfspaces (with intercept) as base hypotheses, we have $d^{*} \leq d = r + 1$.

Furthermore, $m_{0}$, appearing in the running time of \cref{thm:main} is usually \citep{KalaiK09,GhaiS24,GhaiS25,Us25} thought of as being of order $\tilde{O}\Br[\big]{(d + \ln(1/\delta_{0}))/\eps_{0}^{2}}$, where $ \delta_{0} $ and $ \eps_{0} $ are different from the failure and error parameter $ \delta $ and $ \eps $ of the final classifier.
Thus it is crucial for the running time of \cref{alg:main}, that \cref{alg:main} allows for $ \delta_{0} $ and $ \eps_{0} $ being large. Like the algorithm by \citet{Us25}, our algorithm maintains the desirable property of being able to leverage any non-trivial agnostic weak learner ($\gamma_{0} > \eps_{0}$), allowing for $ \eps_{0} $ and $ \delta_{0} $ being thought of as (large) constants.
This is in contrast with the previous works \citet{KalaiK09,GhaiS24,GhaiS25,Feldman10} which all require constraints on the advantage parameter $ \eps_{0} $ depending on $\eps$.

\section{Related works}\label{sec:related}

Among the several works on agnostic boosting, the literature focuses on different aspects of the problem.
While the works of \citet{Ben-DavidLM01,Gavinsky03,KalaiS05,KalaiMV08,LongS05,LongS08,ChenBC16,BrukhimCHM20} are definitely noteworthy, they are not directly related to ours, so we do not discuss them further.

The works most closely related to ours are those by \citet{GhaiS24,GhaiS25,BrukhimCHM20,KalaiK09,Feldman10}.
Some of these employ definitions of weak learner that differ from ours, making the comparisons more nuanced.
In the following, we omit some minor details to allow for a more fluid and yet fair comparison of the results, and we refer the reader to \citet[Appendix~A]{Us25} for a thorough discussion on this matter.
In particular, since some of the works mentioned do not make the dependence on the VC dimension of the base class explicit, we disregard it in the following.

\citet{KalaiK09} propose an efficient agnostic boosting algorithm with sample complexity of order $\tilde{O}\Br[\big]{1/(\gamma_{0}^{4}\eps^{4})}$.
To ensure that their algorithm achieves an error of $\eps$, they have to set $\eps_{0} = O(\gamma_{0}\eps)$.
We remark that the weak learning guarantee in \citet{KalaiK09} is slightly different from \cref{def:agnosticweaklearner}, since it only requires the weak learner to be distribution-specific as their proposed method is based on re-labeling.

\citet{Feldman10} employs a definition of weak learner that differs from ours more substantially:
for $\alpha \ge \gamma > 0$ and a distribution $\dD$ over $\InputSpace$, they consider a weak learner that, given any re-labeling function $g\colon \InputSpace \to [-1,1]$, whenever $\inf_{f \in \fF} \Prob_{(\rx,\ry) \sim (\dD \times g)}[\big]{f(\rx) = \ry} \le 1/2 - \alpha$, produces a hypothesis $w$ such that $\Prob_{(\rx,\ry) \sim (\dD \times g)}[\big]{w(\rx) \ne \ry} \le 1/2 - \gamma$, where the distribution $\dD \times g$ is such that a point $\rx$ is sampled from $\dD$ and then with probability $(1 + g(\rx))/2$ is set to $+1$, and otherwise set to $-1$.
Given such a weak learner, the agnostic learner proposed by \citet{Feldman10} needs $\alpha = O(\eps)$ and $\tilde{O}(1/\eps^{4} + 1/\gamma^{4} + m_{0}/\gamma^{2})$ samples to guarantee error at most $\eps$, where $m_{0}$ is the sample complexity of their weak learner.

\citet{GhaiS24} first employ the definition of weak learner adopted in this work.
They propose an efficient agnostic boosting method with sample complexity of order $\tilde{O}\Br[\big]{1/(\gamma_{0}^{3}\eps^{3})}$, that to guarantee error at most $\eps$ needs $\eps_{0}$ to be of order $O(\eps)$.
The follow-up work by \citet{GhaiS25} proposes an efficient algorithm that leverages unlabeled data.
They show that $\tilde{O}\Br[\big]{1/(\eps^{3}\gamma_{0}^{3})}$ unlabeled and $\tilde{O}\Br[\big]{1/(\eps^{2}\gamma_{0}^{2})}$ labeled samples, having $\eps_{0} = O(\eps)$, suffice for their algorithm to achieve error at most $\eps$.

The algorithm proposed by \citet{Us25} has essentially the same statistical guarantees as ours, hence it achieves near-optimal sample complexity of $\tilde{O}\Br[\big]{1/(\gamma_{0}^{2}\eps^{2})}$.
However, unlike the methods mentioned above, which have a running time that is polynomial in all parameters, their algorithm has a running time that is exponential in the number of samples.
We remark that the bounds by \citet{Us25} are more general than ours, allowing for weak learners with base classes of real-valued hypotheses of finite fat-shattering dimension.
Moreover, their method does not require knowledge of $\theta$, whereas ours and the algorithms mentioned above require that to obtain the stated sample complexities.

\section{Notation and Preliminaries}\label{sec:notation}

We denote by $\N$ the set of positive integers, and, given $n \in \N$, we define $[n] = \{1, \ldots, n\}$.
We define $\{\pm 1\} = \{-1, +1\}$.
For sets $A$ and $B$, we let $B^{A}$ denote the set of all functions from $A$ to $B$.
We denote the set of all probability distributions over $\sA$ by $\DistribOver(\sA)$, the uniform distribution over $\sA$ by $\Uniform(\sA)$ whenever $A$ is finite, and the set of all finite sequences of elements of $\sA$ by $\sA^* = \bigcup_{n=0}^\infty \sA^n$.
We define $\Ln(x) = \ln\Br[\big]{\max\Set{x, e}}$, and $\sign(x) = \indicator{x \geq 0} - \indicator{x < 0}$, so that $\sign(0) = 1$.
With a slight abuse of notation, we extend this notation to real-valued functions and sets of such functions by letting $\sign(f) = \sign \circ f$ and $\sign(\fF) = \Set{\sign(f) \given f \in \fF}$.
Given $f, g \in \mathbb{R}^{\InputSpace}$ and $\alpha, \beta \in \R$, we write $(\alpha f + \beta g)(x) = \alpha f(x) + \beta g(x)$.
Let $n, \ell \in \N$, for any sequence $\sS = (s_1, \ldots, s_n)$ and any vector of indices $I \in [n]^\ell$, we define $\sS\restriction{I} \coloneqq (s_{I_1}, \ldots, s_{I_\ell})$, the subsequence of $\sS$ indexed by $I$.
Throughout, we use boldface letters to denote random variables, .e.g., $\rS$ is a random training sequence and $S$ is a deterministic training sequence.

For simplicity, we always assume that the input space $\InputSpace$ is countable.
We say that $\fH \subseteq \Set{\pm 1}^{\InputSpace}$ \emph{shatters} a set $\Set{x_1, \ldots, x_d} \subseteq \InputSpace$ iff for every possible labeling $\vy \in \Set{\pm 1}^{d}$ there exists $h_{\vy} \in \fH$ such that $h_{\vy}(x_i) = y_i$ for all $i \in [d]$.
The \emph{Vapnik-Chervonenkis (VC) dimension} of $\fH$, denoted by $\VCdim(\fH)$, is the size of the largest set shattered by $\fH$, or $\infty$ if $\fH$ shatters arbitrarily large sets.
This is a standard measure of complexity for classes of binary classifiers.
We also consider the VC dimension of dual classes, defined as follows.

\begin{definition}[Dual class]\label{def:dual-vc}
  Let $\fH \subseteq \Set{\pm 1}^{\InputSpace}$.
  We define \emph{dual} of $\fH$ as
  \begin{align}
    \fH^{*}
    \coloneqq \Set[\big]{h_x\colon \fH \to \Set{\pm 1} \text{ given by } h_x(h) = h(x) \given h \in \fH,\, x \in \InputSpace}
    .
  \end{align}
\end{definition}
The \emph{dual VC dimension} of $\fH$ is defined as $\VCdim(\fH^{*})$.

Given $T \in \N$ and $\fF \subseteq \mathbb{R}^{\InputSpace}$, we let $\fF^{(T)}$ be the class of $T$-wise averages of $\fF$.
That is,
\begin{align}
  \fF^{(T)}
  \coloneqq \Set[\bbbig]{\Frac{1}{T} \sum_{t \in [T]} f_t \given f_t \in \fF \text{ for all } t \in [T]}
  .
\end{align}
Let $\dD \in \DistribOver(\InputSpace \times \Set{\pm 1})$.
For $h \in \Set{\pm 1}^{\InputSpace}$ we define
\begin{align}
  \err_{\dD}(h)
  \coloneqq \Prob_{(\rx, \ry) \sim \dD}[\big]{h(\rx) \ne \ry}
  \qquad\text{and}\qquad
  \corr_{\dD}(h)
  \coloneqq \Ev_{(\rx, \ry) \sim \dD}[\big]{h(\rx) \cdot \ry}
  .
\end{align}
For $\lambda \ge 0$ and $g\in \R^{\InputSpace}$, we define the $\lambda$-margin loss of a $g$ with respect to $\dD$ as
\begin{align}
  \Loss_{\dD}^\lambda(g)
  &\coloneqq \Prob_{(\rx, \ry) \sim \dD}[\big]{\ry \cdot g(\rx) \le \lambda}
  ,
\end{align}
with the shorthand $\Loss_{\dD}(g) \coloneqq \Loss_{\dD}^0(g)$.
Note that $\err_{\dD}(\sign(g)) = \Loss_{\dD}(\sign(g))$.
For a sequence $\sS \in (\InputSpace \times \Set{\pm 1})^*$, we slightly abuse the notation and write $\err_{\sS}(\placeholderarg)$, $\corr_{\sS}(\placeholderarg)$, and $\Loss^\lambda_{\sS}(\placeholderarg)$ in place of $\err_{\Uniform(\sS)}(\placeholderarg)$, $\corr_{\Uniform(\sS)}(\placeholderarg)$, and $\Loss_{\Uniform(\sS)}^\lambda(\placeholderarg)$, respectively.

Finally, given $\dD \in \DistribOver(\InputSpace \times \Set{\pm 1})$ and finite $\fH \subseteq \R^\InputSpace$, we let $\argmin_{h \in \fH} \ls_{\dD}(h) \coloneqq \Set{h^* \in \fH \given \ls_{\dD}(h^*) = \min_{h \in \fH} \ls_{\dD}(h)}$.

\section{Algorithm and Overview of Analysis}
\label{sec:proofsketch}

In this section, we present our algorithm and discuss the ideas behind its formal guarantees.
We sketch the proof of \cref{thm:main}, deferring the rigorous demonstration to \cref{sec:proof}.
For simplicity, we assume $n$ to be even and, in this section, we omit most logarithmic factors.

\begin{algorithm2e}[ht]
  \DontPrintSemicolon
  \caption{Agnostic Boosting Algorithm}\label{alg:main}
  \SetKwInOut{Input}{Input}\SetKwInOut{Output}{Output}
  \Input{Training sequence $\sS = \Br[\big]{(x_1, y_1), \ldots, (x_n, y_n)}$,
    failure probability $\delta$,
    agnostic weak learner $\Weaklearner$,
    its failure probability $\delta_0$,
    sample complexity $m_0$,
    margin parameter $\theta = (\gamma_0 - \eps_0)/2$,
    and dual VC dimension $d^{*}$ of the base class underlying $\Weaklearner$
  }
  $R \gets \Ceil{\Frac*{\ln(n)}{\theta^{2}}}$ \tcp*{Number of rounds}
  $M \gets \Ceil[\bbig]{\Frac*{\ln(\Frac{5R}{\delta})}{\ln(\Frac{1}{\delta_{0}})}}$ \tcp*{Number of hypotheses per round}
  $T \gets \Ceil*{\min\Set*{R,\, 260^2\Frac*{4d^{*} + 2}{\theta^{2}}}}$ \tcp*{Number of hypotheses in combination}
  $\sS_{1} \gets \Br[\big]{(x_1, y_1), \ldots, (x_{n/2}, y_{n/2})}$ \tcp*{First half of training sequence}
  $\sS_{2} \gets \Br[\big]{(x_{n/2+1}, y_{n/2+1}), \ldots, (x_n, y_n)}$ \tcp*{Second half of training sequence}
  \For{$r \gets 1$ \KwTo $R$}{ \label{line:main:for-R}
    $\fB_r \gets \emptyset$ \\
    \For{$I \in [n/2]^{m_{0}}$}{ \label{line:main:for-I}
      \For{$m \gets 1$ \KwTo $M$}{ \label{line:main:for-M}
        $\rb_{r,m,I} \sim \dU$ \tcp*{Sample random seed}
        $\rh_{r,m,I} \gets \Weaklearner(\sS_{1}\restriction{I}, \rb_{r,m,I})$ \label{line:main:call-W} \\
        $\rfB_r \gets \rfB_r \cup \Set{\rh_{r,m,I}}$ \label{line:main:add-to-B} \\
      }
    }
  }
  $\rfB \gets \cup_{r=1}^{R} \rfB_r$ \\
  \tcp{Return a minimizer of $\ls_{\sS_2}$ among all sign of averages of $T$ hypotheses in $\rfB$}
  \Return Classifier $\rv \in \argmin\Set[\big]{\ls_{\sS_{2}}(\rh) \given \rh \in \sign(\rfB^{(T)})}$ \label{line:main:return} \\
\end{algorithm2e}

Given data sequence $\rsS \sim \dD^n$,
\Cref{alg:main} starts by splitting the training data into two equal parts $\rsS_{1}$ and $\rsS_{2}$ (so, throughout the paper, we assume that $n$ is even).
This reflects the two-phase nature of the algorithm:
first, it generates a diverse set of hypotheses by invoking the weak learner on various subsamples of $\rsS_{1}$;
second, it identifies the best combination of these hypotheses according to their performance on $\rsS_{2}$, which serves as a validation set.

More precisely, we show that the set of hypotheses generated in the first phase is such that, with high probability, there exists a combination of $T$ of these hypotheses that achieves a near-optimal error, where by ``combination'' we mean the sign of the average of the $T$ hypotheses.
To see this, let $\fs$ be a near-optimal hypothesis in $\fF$, meaning that $\ls_{\dD}(\fs) \le \errst + 1/n$.
Letting $\dD_{\fs}$ be the conditional distribution of $\dD$ given that $\fs(\rx) = \ry$,
the law of total probability implies that we can decompose the loss of any $v\colon \InputSpace \to \Set{\pm 1}$ as follows:
\begin{align}
  \ls_{\dD}(v)
  &\le \ls_{\dD}(\fs) + \ls_{\dD_{\fs}}(v) \cdot \Prob_{(\rx, \ry) \sim \dD}{\fs(\rx) = \ry}
  \\&\le \errst + \frac{1}{n} + \ls_{\dD_{\fs}}(v) \cdot \Prob_{(\rx, \ry) \sim \dD}{\fs(\rx) = \ry}
  ,
  \label{eq:sketch:decomposition}
\end{align}
where the last inequality follows from the choice of $\fs$.
Thus, if we can find $v_\mathrm{g} \in \sign(\fB^{(T)})$ such that $\ls_{\dD_{\fs}}(v_\mathrm{g})\Prob_{(\rx, \ry) \sim \dD}{\fs(\rx) = \ry} = \tilde{O}(d'/n)$, where $d' = \VCdim\Br[\big]{\sign(\fH^{(T)})} = \tilde{O}(dT)$ (\cref{lem:RcombinationVC}),
we can conclude that
\begin{align}
  \ls_{\dD}(v_\mathrm{g})
  \le \errst + \tilde{O}\Br*{\frac{d'}{n}}
  .
  \label{eq:sketch:start}
\end{align}
For comparison, \cref{thm:main} states that the output $v$ of \Cref{alg:main} satisfies that
\begin{align}
  \ls_{\dD}(v)
  \le \errst + \tilde{O}\Br[\bbbig]{\sqrt{\Frac{\errst \cdot d'}{n}} + \Frac{d'}{n}}
  .
  \label{eq:sketch:goal}
\end{align}
To bridge the gap between the last two inequalities, we first employ a uniform convergence argument (\cref{lem:maurerandpontil}) over $\sign(\fH^{(T)})$ to establish that
\begin{align}
  \ls_{\dD}(v)
  &= \ls_{\rsS_{2}}(v) + \tilde{O}\Br[\bbbig]{\sqrt{\frac{\ls_{\rsS_{2}}(v) d'}{n}} + \frac{d'}{n}}
  .
\end{align}
Moreover, as $v$ is the empirical risk minimizer on $\rsS_{2}$ over $\sign(\fB^{(T)})$, which contains $v_\mathrm{g}$,
\begin{align}
  \ls_{\rsS_{2}}(v)
  &\le \ls_{\rsS_{2}}(v_\mathrm{g})
  .
\end{align}
Next, a Bernstein deviation bound (\cref{lem:shaisbernstein}) implies that
\begin{align}
  \ls_{\rsS_{2}}(v_\mathrm{g})
  &= \ls_{\dD}(v_\mathrm{g}) + \tilde{O}\Br[\bbbig]{\sqrt{\frac{\ls_{\dD}(v_\mathrm{g}) d'}{n}} + \frac{d'}{n}}
  .
\end{align}
Finally, since $v_\mathrm{g}$ is close to optimal (\cref{eq:sketch:start}), we have that
\begin{align}
  \ls_{\dD}(v_\mathrm{g})
  &= \errst + \tilde{O}\Br[\bbbig]{\sqrt{\frac{\ls_{\dD}(f) d'}{n}} + \frac{d'}{n}}
  .
\end{align}
Chaining these inequalities yields \cref{eq:sketch:goal}, our goal.

Altogether, the above reduces the problem to showing the existence of $v_\mathrm{g} \in \sign(\fB^{(T)})$ such that $\ls_{\dD_{\fs}}(v) \cdot \Prob_{(\rx, \ry) \sim \dD}{\fs(\rx) = \ry}$ is of order $\tilde{O}(d'/n)$.
Since this is immediate if $\Prob_{(\rx, \ry) \sim \dD}{\fs(\rx) = \ry} \le O\Br[\big]{\ln(1/\delta)/n}$, we assume otherwise and proceed to show the existence of $v_\mathrm{g}\in\sign(\fB^{(T)})$ such that $\ls_{\dD_{\fs}}(v_\mathrm{g}) = \tilde{O}\Br[\big]{d'/(n\Prob_{(\rx, \ry) \sim \dD}{\fs(\rx) = \ry})}$.
To do so, we start by proving that the set of hypotheses $\fB$ is sufficiently rich to simulate the behavior of \methodname{AdaBoost}{\footnotemark} running on input sequence $\rsS_{\fs} \coloneqq \Set{(x, y) \in \rsS_{1} \given \fs(x) = y}$.%
\footnotetext{More accurately, since we consider the average of hypotheses, we employ a variation of \methodname{AdaBoost} that outputs a uniformly weighted majority of weak hypotheses (cf.\ \Cref{alg:adaboost}).
  This is in contrast with the standard weighted majority voting of \methodname{AdaBoost}.
}
The goal is to exploit the fact that, in the realizable setting, running \methodname{AdaBoost} for sufficiently many rounds yields a classifier with large margins on the training data.
More concretely, this would allow us to construct a classifier that achieves zero $(\theta/2)$-margin loss on $\rsS_{\fs}$ with $\theta = (\gamma_{0} - \eps_{0})/2$.
That is, as long as we can simulate the realizable setting by providing \methodname{AdaBoost} with hypotheses with advantage $\theta$ each time it calls the weak learner, it will produce a classifier $v_\mathrm{ada}$ such that $\ls_{\rsS_{\fs}}^{\theta/2}(v_\mathrm{ada}) = 0$.

We argue that we can indeed find a suitable hypothesis in $\fB$ for each call to the weak learner by noticing that, since $\fs \in \fF$, the agnostic weak learning guarantee ensures that for any $\dD'\in \DistribOver(\rsS_{\fs})$ we have that
\begin{math}
  \Prob_{\rsS \sim \dD'^{m_{0}}, \rb}{\corr_{\dD'}(\Weaklearner(\rsS, \rb))
  \ge \gamma_{0} - \eps_{0}} \ge 1 - \delta_{0}
  ,
\end{math}
or, using that $\corr_{\dD'}(h) = 1 - 2 \ls_{\dD'}(h)$ for $h \in \Set{\pm 1}^\InputSpace$, that
\begin{align}
  \Prob_{\rsS \sim \dD'^{m_{0}}, \rb}{\ls_{\dD'}(\Weaklearner(\rsS, \rb)) \le 1/2 - \theta}
  \ge 1 - \delta_{0}
  .
\end{align}
Since $\rb$ and $\rsS$ are independent, this also implies that for any $\dD' \in \DistribOver(\rsS_{\fs})$ there exists a subsample $S' \in \rsS_{\fs}^{m_{0}}$ such that
\begin{align}
  \Prob_{\rb}{\err_{\dD'}(\Weaklearner(S', \rb)) \le 1/2 - \theta}
  \ge 1 - \delta_{0}
  .
\end{align}
Moreover, as $\rsS_{\fs}^{m_{0}} \subseteq\rsS_{1}^{m_{0}}$ and the \KwSty{for} loop at \cref{line:main:for-I} ranges over all $\rsS_{1}^{m_{0}}$, some iteration of \cref{line:main:call-W} will consider $S'$.
Finally, we amplify this probability by calling the weak learner $M = \Theta(\ln(R/\delta)/\ln(1/\delta_{0}))$ times on each subsample of size $m_{0}$, ensuring a high probability that we can successfully run \methodname{AdaBoost} on $\rsS_{\fs}$ for $R$ rounds while simulating the realizable case by using hypotheses from $\fB$.

The argument so far shows that if we were to set $T = R$ and explore all hypotheses in $\sign(\fB^{(T)})$, we would likely find a good classifier $v_\mathrm{ada}$ with zero $(\theta/2)$-margin loss on $\rsS_{\fs}$.
However, the size of $\sign(\fB^{(T)})$ is exponential in $T$, and $R$ can be as large as $O(\ln(n)/\theta^{2})$.
This is why we argued for $v_\mathrm{ada}$ having zero $(\theta/2)$-margin loss when simply having zero loss ($0$-margin) would suffice to obtain \cref{eq:sketch:goal}.
We leverage the extra margin to show that $v_\mathrm{ada}$ can be ``pruned'' down to a combination of $R^* = \tilde{O}(d^{*}/\theta^{2})$ hypotheses from $\fB$ and still obtain zero loss on $\rsS_{\fs}$, where $d^{*}$ is the dual VC dimension of $\fB$.
Thus, there exists $v_\mathrm{pruned} \in \sign(\fB^{(R^*)})$ such that $\ls_{\rsS_{\fs}}(v_\mathrm{pruned}) = 0$.
Accordingly, we set $T = \Ceil{\min\{R^*, R\}} = \tilde{O}\Br[\big]{\min\Set{d^*, \ln(n)}/\theta^{2}}$ to ensure that, with high probability, $\sign(\fB^{(T)})$ contains a hypothesis $v_\mathrm{g}$ such that $\ls_{\rsS_{\fs}}(v_\mathrm{g}) = 0$, be it $\sign(v_\mathrm{pruned})$ or $\sign(v_\mathrm{ada})$.

From here,
as we are considering the case where $\Prob_{(\rx, \ry) \sim \dD}{\fs(\rx) = \ry} = \Omega\Br[\big]{\ln(1/\delta)/n}$, a Chernoff bound implies that, with high probability, $\abs{\rsS_{\fs}} = \Theta(n \Prob_{(\rx, \ry) \sim \dD}{\fs(\rx) = \ry})$.
Finally, applying a realizable uniform convergence bound on $\sign(\fH^{(T)})$ yields that $\ls_{\dD_{\fs}}(v_\mathrm{g})$ is of order $\tilde{O}\Br[\big]{d'/(n\Prob_{(\rx, \ry) \sim \dD}{\fs(\rx) = \ry})}$, concluding the proof.

\section{Conclusion}\label{sec:conclusion}

In this work, we presented an agnostic boosting algorithm that achieves near-optimal sample complexity and a running time polynomial in the sample size, when the other parameters of the problem are considered fixed.
Namely, we show that the running time of \cref{alg:main} is polynomial in the sample size.
An interesting question is whether there exists a statistically optimal (up to logarithmic factors) algorithm whose running time is fully polynomial in all parameters.

\acks{While this work was carried out, Andrea Paudice was supported by the Novo Nordisk Foundation Start Package Grant No. NNF24OC0094365 (Actionable Performance Guarantees in Machine Learning).

While this work was carried out,
Mikael M{\o}ller H{\o}gsgaard was supported by an Internationalisation Fellowship from the Carlsberg Foundation.
Furthermore,
Mikael M{\o}ller H{\o}gsgaard and Arthur da Cunha were supported by the European Union (ERC, TUCLA, 101125203).
Views and opinions expressed are however those of the author(s) only and do not necessarily reflect those of the European Union or the European Research Council.
Neither the European Union nor the granting authority can be held responsible for them.
Lastly, Mikael M{\o}ller H{\o}gsgaard and Arthur da Cunha were also supported by Independent Research Fund Denmark (DFF) Sapere Aude Research Leader grant No.\ 9064-00068B.
}

\bibliography{ref.bib}

\appendix
\section{Formal proof of \texorpdfstring{\Cref{thm:main}}{main Theorem}}\label{sec:proof}
In this section, we prove the following theorem, which readily implies \Cref{thm:main}.

\begin{theorem}\label{thm:main-full}
    Given $\gamma_{0}, \eps_{0}, \delta_{0} \in (0, 1)$,
    $m_{0} \in \N$,
    and $\fH, \fF \subseteq \Set{\pm 1}^{\InputSpace}$,
    let $\theta = (\gamma_{0} - \eps_{0})/2$,
    let $d$ be the VC dimension of $\fH$ and $d^{*}$ its dual VC dimension,
    and let $\Weaklearner$ be a $(\gamma_{0}, \eps_{0}, \delta_{0}, m_{0})$ agnostic weak learner for $\fF$ with base class $\fH$.
    Then,
    for any $\delta \in (0, 1)$,
    $n \in \N$,
    and distribution $\dD$ over $\InputSpace \times \Set{\pm 1}$,
    given $\rsS \sim \dD^{n}$,
    \Cref{alg:main}, when executed on input $(\rsS, \delta, \Weaklearner, \delta_{0}, m_{0}, \theta, d^*)$,
    returns,
    with probability at least $1 - \delta$ over $\rsS$ and the randomness of $\Weaklearner$,
    a classifier $v\colon \InputSpace \to \Set{\pm 1}$ such that
    \begin{align}
        \err_{\dD}(v)
        &\le
            \errst
            + 10\sqrt{\errst \cdot \Frac{d'\Ln\Br{\Frac*{10en}{d'}} + \ln\Br{\Frac*{5}{\delta}}}{n}}
            + 182 \cdot \Frac{d'\Ln\Br{\Frac*{10en}{d'}} + 4\ln\Br{\Frac*{5}{\delta}}}{n}
            ,
        \label{eq:main-full--err-bound}
    \end{align}
    where $\errst = \inf_{f \in \fF} \ls_{\dD}(f)$,
    and $d' = 4Td\ln\Br{4eT}$,
    with $T = \Ceil{\min\Set{\Frac{\ln\Br{n}}{\theta^{2}},\, 260^2\Frac[\Br]{4d^{*} + 2}{\theta^{2}}}}$.

    Moreover,
    whenever the above bound is not vacuous,
    i.e., for $n = \Omega(\max\Set{dT,\, \ln\Br{1/\delta}})$,
    \cref{alg:main} invokes $\Weaklearner$ at most
    \begin{math}
        O\Br{n^{m_{0} + 3}}
    \end{math}
    times,
    and the running time of \cref{alg:main} is at most
    \begin{align}
        \operatorname{Eval}_{\fH}(1) \cdot n^{O\Br*{m_{0} \cdot \min\{d^{*},\, \ln\Br{n}\}/\theta^2}}
        ,
    \end{align}
    where $\operatorname{Eval}_{\fH}(1)$ is the time it takes to evaluate a single hypothesis from $\fH$ on a single point.
\end{theorem}

The proof leverages a few known lemmas.
Some of the statements differ slightly from the original, so, for completeness, we provide proofs for these in the \cref{app:auxiliary}.

The first lemma is a uniform error bound on a hypothesis class with bounded VC dimension, which also encompasses the empirical loss for each hypothesis.
This lemma follows from \citet[Theorem~6]{Maurer2009EmpiricalBB} by using the fact that an $\ell_{\infty}$-cover of the hypothesis class for any precision over $n$ points has size at most $(en/d)^{d}$, by the Sauer-Shelah lemma.

\begin{lemma}[Corollary of \protect{\citet[Theorem~6]{Maurer2009EmpiricalBB}}]\label{lem:maurerandpontil}
    Let $n \in \N$,
    $\delta \in (0, 1)$,
    and $\dD \in \DistribOver(\InputSpace \times \{\pm 1\})$,
    and let $\fH \subseteq \Set{\pm 1}^\InputSpace$ have VC dimension $d$.
    Then, it holds with probability at least $1 - \delta$ over $\rsS \sim \dD^{n}$ that for all $h \in \fH$,
    \begin{align}
        \ls_{\dD}(h)
        &\leq \ls_{\rsS}(h) + \sqrt{\frac{18\ls_{\rsS}(h)\Br[\big]{d\Ln(20en/d) + \ln(1/\delta)}}{n}} + \frac{15\Br[\big]{d\Ln(20en/d) + \ln(1/\delta)}}{n}
        .
    \end{align}
\end{lemma}

The next lemma is a version of the previous one that holds only for a single fixed hypothesis, but yields better concentration.

\begin{lemma}[\citet{UML}, Lemma B.10]\label{lem:shaisbernstein}
    Let $n \in \N$, $\delta \in (0, 1)$, and $\dD \in \DistribOver(\InputSpace \times \{\pm 1\})$.
    Then, it holds for any fixed $g\colon \InputSpace \to \{\pm 1\}$ that with probability at least $1 - \delta$ over $\rsS \sim \dD^{n}$,
    \begin{align}
        \ls_{\rsS}(g)
        &\leq \ls_{\dD}(g) + \sqrt{\frac{2\ls_{\dD}(g)\ln(1/\delta)}{3n}} + \frac{2\ln(1/\delta)}{n}
        .
    \end{align}
\end{lemma}

Next, we provide guarantees on \cref{alg:adaboost}, which is a version of \methodname{AdaBoost} that aggregates hypotheses via simple averaging (equal weighting).
This is in contrast with classical \methodname{AdaBoost}, which weights hypotheses according to their empirical error.
Also, notice that \cref{alg:adaboost} is not called by \cref{alg:main}.
Hence, it does not actually run and serves only as a theoretical device for the analysis.

\begin{algorithm2e}[ht]
  \DontPrintSemicolon
  \caption{Boosting Algorithm}\label{alg:adaboost}
  \SetKwInOut{Input}{Input}\SetKwInOut{Output}{Output}
  \Input{Training sequence $S = \Br[\big]{(x_{1},y_{1}), \ldots, (x_{n},y_{n})}$ and edge parameter $\theta$.}
  $R \ge \Ceil[\big]{\ln(2n)/\theta^{2}}$ \tcp*{Number of rounds}
  $D_{1}\gets (1/n, 1/n, \ldots, 1/n) \in [0,1]^{n}$ \label{line:adaboost:initD} \tcp*{Initial distribution (uniform)}
  $\alpha \gets \frac{1}{2} \ln\Br[\bbig]{\frac{1/2+\theta}{1/2-\theta}}$ \tcp*{Weighting factor}
  \For{$r \gets 1$ \KwTo $R$}{
     Receive a hypothesis $h_{r} \colon \InputSpace \to \Set{\pm 1}$ such that $\ls_{D_{r}}(h_{r}) \le 1/2 - \theta$ \label{line:adaboost:get-weak-hypothesis}\\
     $Z_{r} \gets \sum_{i=1}^{n} D_{r,i} \cdot \exp\Br[\big]{-\alpha y_{i} h_{r}(x_{i})}$ \label{line:adaboost:Z} \tcp*{Normalization factor}
     $D_{r+1,i} \gets D_{r,i} \cdot \exp\Br[\big]{-\alpha y_{i} h_{r}(x_{i})}/Z_{r}$ for $i = 1, \ldots, n$ \label{line:adaboost:newD} \tcp*{New distribution}
  }
  \Return Voting classifier $\frac{1}{R}\sum_{r=1}^{R} h_{r}$ \label{line:adaboost:return}
\end{algorithm2e}

Our next lemma states that if \Cref{line:adaboost:get-weak-hypothesis} always outputs a weak classifier with $\theta$-margin, then the final classifier also has margin of order $\Theta(\theta)$ on the whole input dataset.

\begin{lemma}\label{lem:adaboostconvergence}
    Let $n \in \N$ and $\theta \in (0, 1/2)$.
    In the execution of \Cref{alg:adaboost} on inputs $\sS \in (\InputSpace \times \Set{\pm 1})^{n}$ and $\theta$,
    if \cref{line:adaboost:get-weak-hypothesis} succeeds each time, then for any integer $R \ge \Ceil{\ln(2n)/\theta^{2}}$, we have that the output classifier $h = \sum_{r = 1}^{R} h_{r}/R$ satisfies that $\ls_{\sS}^{\theta/2}(h) = 0$.
\end{lemma}

The next lemma asserts about voting classifiers composed of hypotheses from a class with bounded dual VC dimension.
It states that if such a classifier has large margins on a set of examples, then it can be pruned down to a classifier that is correct on those examples, where, crucially, the number of hypotheses in this pruned voting classifier is independent of the number of points in the training sequence.

\begin{lemma}\label{lem:dualVC-dimensionofcombination}
    Let $n \in \N$,
    $\theta > 0$,
    and $\sS = \Br[\big]{(x_{1}, y_{1}), \ldots, (x_{n}, y_{n})}$,
    and given $\fH \subseteq \Set{\pm 1}^\InputSpace$ with dual VC dimension $d^{*}$,
    let $v = \sum_{h \in \fH} \alpha_{h}h$ with $\sum_{h \in \fH} \alpha_{h} = 1$ and $\alpha_{h} \geq 0$ for all $h \in \fH$.
    If $v(x)y \geq \theta$ for all $(x, y) \in \sS$,
    then for $L = \Ceil{130^2(4d^{*} + 2)/\theta^{2}}$,
    there exists $h_{1}, \ldots, h_{L} \in \fH$ such that $\tilde{v} = \frac{1}{L}\sum_{i = 1}^{L} h_{i}$ satisfies that $\tilde{v}(x)y > 0$ for all $(x, y) \in \sS$, and $\alpha_{h_{i}} > 0$ for $i \in [L]$.
\end{lemma}

Above, having every $\alpha_{h_{i}}$ strictly positive ensures that $\tilde{v}$ is a ``pruning'' of $v$, consisting of hypotheses present in the convex combination associated with $v$.

Lastly, the next lemma states that if a class $\fH$ has VC dimension $d$, then the class of $T$-wise averages of hypotheses from $\fH$ has VC dimension at most $O(Td\log(T))$.

\begin{lemma}\label{lem:RcombinationVC}
    Let $\fH \subseteq \Set{\pm 1}^\InputSpace$, and $T \in \N$.
    If $\fH$ has VC dimension $d$,
    then the class $\sign(\fH^{(T)}) = \Set{\sign(g) \given g = \tfrac{1}{T}\sum_{i = 1}^{T} h_{i} \text{ with } h_{1}, \ldots, h_{T} \in \fH}$ has VC dimension at most $4Td\ln(4eT)$.
\end{lemma}

With the lemmas in place, we now prove \Cref{thm:main}.

\begin{proof}[of \Cref{thm:main}]
    We first analyze the computational complexity of the algorithm, then its sample complexity.

\paragraph{Computational Complexity Analysis:}
    The weak learner is called $R \cdot M \cdot (n/2)^{m_{0}}$ times within the three \KwSty{for} loops (starting at \cref{line:main:for-R} and ending at \ref{line:main:add-to-B}).
    Since $R = O(\ln(n)/\theta^{2})$,
    and
    \begin{align}
        M
        &= O\Br[\bbbig]{\Frac{\ln(R/\delta)}{\ln(1/\delta_{0})}}
        \\&= O\Br[\bbbig]{\Frac{\ln\Br*{\ln(n)/(\delta\theta^{2})}}{\ln(1/\delta_{0})}}
        ,
    \end{align}
    the number of calls to the weak learner is at most
    \begin{align}
        O\Br[\bbbig]{\frac{\ln(n) \cdot n^{m_{0}} \cdot \ln\Br*{\Frac[][\Br]{\ln(n)}{\delta \theta^{2}}}}{\ln(1/\delta_{0})\theta^{2}}}
        .
    \end{align}
    Furthermore, the search over $\fB^{(T)}$ in \cref{line:main:return} is over at most $\abs{\fB}^{T} = (R \cdot (n/2)^{m_{0}} \cdot M)^{T}$ combinations.
    For each of these, the empirical loss is calculated, which takes $O(nT) \cdot \operatorname{Eval}_{\fH}(1)$ time, as $\operatorname{Eval}_{\fH}(1)$ is the time needed to evaluate a single hypothesis from $\fH$ on a single point.
    As $T = O(\min\{d^{*},\, \ln(n)\}/\theta^{2})$, this overall step requires at most time
    \begin{align}
        O(nT) \cdot \operatorname{Eval}_{\fH}(1) \cdot \Br[\bbbig]{\frac{\ln(n) \cdot n^{m_{0}} \cdot \ln\Br*{\Frac[][\Br]{\ln(n)}{\delta\theta^{2}}}}{\ln(1/\delta_{0})\theta^{2}}}^{O(\min\{d^{*},\, \ln(n)\}/\theta^{2})}.
    \end{align}

    When $n = \Omega(\max\Set{dT,\, \ln(1/\delta)})$ (the non-vacuous regime of \cref{thm:main-full}),
    we have that $R \cdot M = O(n^{2}\ln(n)) = o(n^{2.1})$.
    So, the $R \cdot M \cdot (n/2)^{m_{0}}$ calls to the weak learner amount to at most
    \begin{math}
        O(n^{m_{0} + 2.1})
        ,
    \end{math}
    and the computational cost of finding the best combination of hypotheses is at most
    \begin{math}
        \operatorname{Eval}_{\fH}(1) \cdot n^{m_{0} \cdot O(\min\{d^{*},\, \ln(n)\}/\theta^{2})}
        .
    \end{math}

\paragraph{Generalization Error Analysis:}
    To shorten the notation, we let $\Gamma(n) = d'\Ln(10en/d')$, where $d' = \VCdim(\sign(\fH^{(T)}))$.

    We will consider the execution of \cref{alg:main} on input as per the statement of \cref{thm:main-full}, so let $\cA(\rsS) \coloneqq \cA(\rsS, \delta, \Weaklearner, \delta_{0}, m_{0}, \theta, d^*)$ denote the associated output classifier.
    Moreover,
    let $\rsS_{1}$ and $\rsS_{2}$ be the first and second halves of $\rsS \sim \dD^{n}$, respectively,
    and let $\rb = \Set{\rb_{r, m, I} \given r \in [R], m \in [M], I \in [n/2]^{m_{0}}}$ be the randomness associated with the weak learner.

    Let $E_{1}$ be the event over $\rsS_{1}$ and $\rb$ that $\sign(\rfB^{(T)})$ contains a classifier with low population loss.
    Precisely, we let
    \begin{align}
        E_{1}
        = \Set[\bbbig]{\min_{v \in \sign(\rfB^{(T)})} \ls_{\dD}(v) \leq \errst + \frac{61\Br[\big]{\Gamma(n) + \ln(\Frac{5}{\delta})}}{n}}
        .
    \end{align}

    Let $E_{2}$ be the event over $\rsS_{1}, \rsS_{2}$ and $\rb$ that there exists $v_\mathrm{g} \in \argmin_{v \in \sign(\rfB^{(T)})}\ls_{\dD}(v)$ whose empirical loss over $\rsS_{2}$ (validation loss) is not much higher than its population loss.
    Precisely, we let
    \begin{align}
        E_{2}
        = \Set[\bbbig]{\exists v_\mathrm{g} \in \argmin_{v \in \sign(\rfB^{(T)})}\ls_{\dD}(v) : \ls_{\rsS_{2}}(v_\mathrm{g}) \leq \ls_{\dD}(v_\mathrm{g}) + \sqrt{\frac{4\ls_{\dD}(v_\mathrm{g})\ln(\Frac{5}{\delta})}{3n}} + \frac{4\ln(\Frac{5}{\delta})}{n}}
        .
    \end{align}

    Lastly, given $v \in \sign(\fH^{(T)})$, let $E_{3, v}$ be the event over $\rsS_{2}$ that population loss of $v$ is not much higher than its validation loss:
    \begin{align}
        E_{3, v}
        = \Set[\bbbig]{\ls_{\dD}(v) \leq \ls_{\rsS_{2}}(v) + \sqrt{\frac{36\ls_{\rsS_{2}}(v)\Br[\big]{\Gamma(n) + \ln(\Frac{5}{\delta})}}{n}} + \frac{30\Br[\big]{\Gamma(n) + \ln(\Frac{5}{\delta})}}{n}}
        ;
    \end{align}
    and we let $E_{3}$ be the simultaneous event for all classifiers in $\sign(\fH^{(T)})$.
    That is, we let
    \begin{align}
        E_{3}
        = \bigcap_{v \in \sign(\fH^{(T)})} E_{3, v}
        .
    \end{align}

    We first argue that if $E_{1}, E_{2}$, and $E_{3}$ occur simultaneously,
    then
    \begin{align}
        \ls_{\dD}\Br[\big]{\cA(\rsS)}
        &\leq \errst + 10\sqrt{\frac{\errst\Br[\big]{\Gamma(n) + \ln(\Frac{5}{\delta})}}{n}} + 182\frac{\Gamma(n) + 4\ln(\Frac{5}{\delta})}{n}
        ,
    \end{align}
    so the generalization error bound in \cref{thm:main-full} readily follows from \cref{lem:RcombinationVC} and the monotonic decrease of $\Ln(x)/x = \ln(\max\Set{x, e})/x$ for $x > 0$.
    After that, we finish the proof by showing that those three events occur simultaneously with probability at least $1 - \delta$.

    Towards this goal, we start by noting that under $E_{3}$, the fact that $\cA(\rsS)$ belongs to $\sign(\rfB^{(T)}) \subseteq \sign(\fH^{(T)})$ implies that
    \begin{align}
        \ls_{\dD}\Br[\big]{\cA(\rsS)}
        \leq \ls_{\rsS_{2}}\Br[\big]{\cA(\rsS)} + \sqrt{\frac{36\ls_{\rsS_{2}}\Br[\big]{\cA(\rsS)}\Br[\big]{\Gamma(n) + \ln(\Frac{5}{\delta})}}{n}} + \frac{30(\Gamma(n) + \ln(\Frac{5}{\delta}))}{n}
        .
        \label{eq:proof:temp:1}
    \end{align}
    Furthermore, since $\cA(\rsS)$ is a minimizer of $\ls_{\rsS_{2}}$ over $\sign(\rfB^{(T)})$ (cf.\ \cref{line:main:return}), it holds that
    \begin{align}
        \ls_{\rsS_{2}}\Br[\big]{\cA(\rsS)}
        \leq \ls_{\rsS_{2}}(v)
        \enspace\text{for all }
        v \in \sign(\rfB^{(T)})
        .
    \end{align}
    In particular, this holds for $v_\mathrm{g} \in \argmin_{v \in \sign(\rfB^{(T)})}\ls_{\dD}(v)$ as in the definition of $E_{2}$.
    Applying this to \cref{eq:proof:temp:1}, we obtain that
    \begin{align}
        \ls_{\dD}\Br[\big]{\cA(\rsS)}
        &\leq \ls_{\rsS_{2}}(v_\mathrm{g}) + \sqrt{\frac{36\ls_{\rsS_{2}}(v_\mathrm{g})\Br[\big]{\Gamma(n) + \ln(\Frac{5}{\delta})}}{n}} + \frac{30(\Gamma(n) + \ln(\Frac{5}{\delta}))}{n}
        .
        \label{eq:proof:temp:2}
    \end{align}
    For the radicand in this expression, the bound on $\ls_{\rsS_{2}}(v_\mathrm{g})$ under $E_{2}$ implies that $\ls_{\rsS_{2}}(v_\mathrm{g}) \leq 2(\ls_{\dD}(v_\mathrm{g}) + 4\ln(\Frac{5}{\delta})/n)$, so
    \begin{align}
        \frac{36\ls_{\rsS_{2}}(v_\mathrm{g})\Br[\big]{\Gamma(n) + \ln(\Frac{5}{\delta})}}{n}
        &\leq \frac{72\Br*{\ls_{\dD}(v_\mathrm{g}) + 4\ln(\Frac{5}{\delta})/n}\Br[\big]{\Gamma(n) + \ln(\Frac{5}{\delta})}}{n}
        \\&\leq \frac{72\ls_{\dD}(v_\mathrm{g})\Br[\big]{\Gamma(n) + \ln(\Frac{5}{\delta})}}{n} + \frac{72\Br[\big]{\Gamma(n) + 4\ln(\Frac{5}{\delta})}^{2}}{n^{2}}
        .
    \end{align}
    Applying this to \cref{eq:proof:temp:2} and using that $\sqrt{a+b} \leq \sqrt{a} + \sqrt{b}$ for $a, b \geq 0$ yields that
    \begin{align}
        \ls_{\dD}\Br[\big]{\cA(\rsS)}
        &\leq \ls_{\rsS_{2}}(v_\mathrm{g}) + \sqrt{\frac{72\ls_{\dD}(v_\mathrm{g})\Br[\big]{\Gamma(n) + \ln(\Frac{5}{\delta})}}{n}}
            +\frac{(\sqrt{72} + 30)(\Gamma(n) + 4\ln(\Frac{5}{\delta}))}{n}
        \\&\leq \ls_{\dD}(v_\mathrm{g}) + \Br[\bbbig]{\sqrt{\frac{4}{3}} + \sqrt{72}} \sqrt{\frac{\ls_{\dD}(v_\mathrm{g})\Br[\big]{\Gamma(n) + \ln(\Frac{5}{\delta})}}{n}}
            \\&\qquad+ \frac{(\sqrt{72} + 34)(\Gamma(n) + 4\ln(\Frac{5}{\delta}))}{n}
        \\&\leq \ls_{\dD}(v_\mathrm{g}) + 10\sqrt{\frac{\ls_{\dD}(v_\mathrm{g})\Br[\big]{\Gamma(n) + \ln(\Frac{5}{\delta})}}{n}} + 42\frac{\Gamma(n) + 4\ln(\Frac{5}{\delta})}{n}
        ,
    \end{align}
    where the second inequality holds since choosing $v_\mathrm{g}$ as in $E_{2}$ implies that $\ls_{\rsS_{2}}(v_\mathrm{g}) \leq \ls_{\dD}(v_\mathrm{g}) + \sqrt{4\ls_{\dD}(v_\mathrm{g})\ln(\Frac{5}{\delta})/(3n)} + 4\ln(\Frac{5}{\delta})/n$.

    Lastly, under $E_1$, the fact that $v_\mathrm{g}$ belongs to $\argmin_{v \in \sign(\fB^{(T)})}\ls_{\dD}(v)$ implies that $\ls_{\dD}(v_\mathrm{g}) \leq \errst + \Frac{61\Br[\big]{\Gamma(n) + \ln(\Frac{5}{\delta})}}{n}$.
    Applying this and that $\sqrt{a+b} \leq \sqrt{a} + \sqrt{b}$ for $a, b \geq 0$ to the above yields that
    \begin{align}
        \ls_{\dD}\Br[\big]{\cA(\rsS)}
        &\leq \errst + 10\sqrt{\frac{\errst\Br[\big]{\Gamma(n) + \ln(\Frac{5}{\delta})}}{n}}
            + (10 \cdot \sqrt{61} + 61 + 42)\frac{\Gamma(n) + 4\ln(\Frac{5}{\delta})}{n}
        \\&\leq \errst + 10\sqrt{\frac{\errst\Br[\big]{\Gamma(n) + \ln(\Frac{5}{\delta})}}{n}} + 182\frac{(\Gamma(n) + 4\ln(\Frac{5}{\delta}))}{n}
        ,
    \end{align}
    as desired.

    Thus, it only remains to prove that $E_{1}, E_{2}$, and $E_{3}$ happen simultaneously with probability at least $1 - \delta$.
    This follows from a union bound after we show that $E_{1}$ happens with probability at least $1 - 3\delta/5$, and $E_{2}$ and $E_{3}$ each happen with probability at least $1 - \delta/5$.

\paragraph{Showing that $E_{1}$ happens with high probability:}
    Recall that we let%
    \begin{align}
        E_{1}
        = \Set[\bbbig]{\min_{v \in \sign(\rfB^{(T)})} \ls_{\dD}(v) \leq \errst + \frac{61\Br[\big]{\Gamma(n) + \ln(\Frac{5}{\delta})}}{n}}
    \end{align}
    be an event over $\rsS_{1}$ and $\rb$.
    We claim that $E_{1}$ happens with probability at least $1 - 3\delta/5$.

    To this end, we set some definitions.
    Let $\fs$ be a function in $\fF$ such that $\ls_{\dD}(\fs) \leq \errst + 1/n$.
    Let $\rsS_{\fs} = \Set{(\rx, \ry) \in \rsS_{1} \given \ry = \fs(\rx)}$ be the subset of examples in $\rsS_{1}$ labeled according to $\fs$.
    Let $\dD_{\fs} \in \DistribOver(\InputSpace \times \{\pm 1\})$ be the distribution obtained from $\dD$ by conditioning on the label being given by $\fs$.
    That is, for any event $A \subseteq \InputSpace \times \{\pm 1\}$, we have $\dD_{\fs}(A) = \p_{(\rx, \ry) \sim \dD}[A, \ry = \fs(\rx)]/\p_{(\rx, \ry) \sim \dD}[\ry = \fs(\rx)]$.
    This ensures observations in $\rsS_{\fs}$ are distributed according to $\dD_{\fs}$.

    We first notice that for any function $g\colon \InputSpace \to \{\pm 1\}$, we have that
    \begin{align}
        \ls_{\dD}(g)
        &= \Prob_{(\rx, \ry) \sim \dD}*{g(\rx) \neq \ry} \tag{by definition of $\ls_{\dD}$}
        \\&= \Prob_{\rx, \ry}*{g(\rx) \neq \ry \given \fs(\rx) \neq \ry} \Prob_{\rx, \ry}*{\fs(\rx) \neq \ry}
        \\&\qquad+ \Prob_{\rx, \ry}*{g(\rx) \neq \ry \given \fs(\rx) = \ry} \Prob_{\rx, \ry}*{\fs(\rx) = \ry}
            \tag{by law of total probability}
        \\&\leq \ls_{\dD}(\fs) + \Prob_{\rx, \ry}*{g(\rx) \neq \ry \given \fs(\rx) = \ry} \Prob_{\rx, \ry}*{\fs(\rx) = \ry}
            \tag{as $\Prob{\placeholderarg} \leq 1$}
        \\&\leq \errst + \frac{1}{n} + \ls_{\dD_{\fs}}(g) \Prob_{\rx, \ry}*{\fs(\rx) = \ry}
            \tag{by definitions of $\ls_{\dD}$ and $\dD_{\fs}$}
        .
    \end{align}
    Thus, if we can argue that there exists $v \in \sign(\fB^{(T)})$ such that $\ls_{\dD_{\fs}}(v) \p_{(\rx, \ry) \sim \dD}[\fs(\rx) = \ry]$ is at most $60(\Gamma(n) + \ln(\Frac{5}{\delta}))/n$, then we are done.

    Let $p = \Prob_{(\rx, \ry) \sim \dD}*{\fs(\rx) = \ry}$.
    If $p \leq 16\ln(\Frac{5}{\delta})/n$, then the bound in $E_{1}$ holds for any $g$ since $\ls_{\dD_{\fs}}(g) \leq 1$.
    We thus assume from now on that $p \geq 16\ln(\Frac{5}{\delta})/n$.
    Let $\rN = \abs{\rsS_{\fs}}$ and let
    \begin{align}
        E_{1, 1}
        = \Set*{\rN \geq \frac{pn}{4}}
    \end{align}
    be an event over $\rsS_{1}$.

    Furthermore, given $v \in \sign(\fH^{(T)})$, let
    \begin{align}
        E_{1, 2, v}
        = \Set*{\ls_{\dD_{\fs}}(v) \leq \ls_{\rsS_{\fs}}(v) + \sqrt{\frac{18\ls_{\rsS_{\fs}}(v)(\Gamma(2\rN) + \ln(\Frac{5}{\delta}))}{\rN}} + \frac{15(\Gamma(2\rN) + \ln(\Frac{5}{\delta}))}{\rN}}
    \end{align}
    be an event over $\rsS_{1}$;
    and let $E_{1, 2}$ be the simultaneous event for all classifiers in $\sign(\fH^{(T)})$.
    I.e.,
    \begin{align}
        E_{1, 2}
        = \bigcap_{v \in \sign(\fH^{(T)})} E_{1, 2, v}
        .
    \end{align}

    Finally, let
    \begin{align}
        E_{1, 3}
        = \Set*{\exists v \in \sign(\rfB^{(T)}): \ls_{\rsS_{\fs}}(v) = 0}
    \end{align}
    be an event over $\rsS_{1}$ and $\rb$.

    If $E_{1, 1}, E_{1, 2}, E_{1, 3}$ occur simultaneously, then for any $v \in \sign(\rfB^{(T)})$ with $\ls_{\rsS_{\fs}}(v) = 0$ (whose existence follows from $E_{1,3}$), it holds that
    \begin{align}
        &\ls_{\dD_{\fs}}(v) \cdot p
        \\&\quad\leq \Br[\bbbbig]{\ls_{\rsS_{\fs}}(v) + \sqrt{\frac{18\ls_{\rsS_{\fs}}(v)\Br*{\Gamma(2\rN) + \ln(\Frac{5}{\delta})}}{\rN}} + \frac{15(\Gamma(2\rN) + \ln(\Frac{5}{\delta}))}{\rN}}p
            \tag{by $E_{1, 2}$}
        \\&\quad= \frac{15(\Gamma(2\rN) + \ln(\Frac{5}{\delta}))}{\rN}p
            \tag{as $\ls_{\rsS_{\fs}}(v) = 0$}
        \\&\quad\leq \frac{60(d'\Ln(5epn/d') + \ln(\Frac{5}{\delta}))}{n}
            \label{eq:proof:E1:temp:1}
        \\&\quad\leq \frac{60(d'\Ln(10en/d') + \ln(\Frac{5}{\delta}))}{n}
            \tag{as $p \leq 1$ and $\Ln$ is non-decreasing}
        \\&\quad= \frac{60(\Gamma(n) + \ln(\Frac{5}{\delta}))}{n}
            \tag{as $\Gamma(n) = d'\Ln(10en/d')$}
        ,
    \end{align}
    where \cref{eq:proof:E1:temp:1} follows from $E_{1, 1}$, the fact that $\Gamma(2\rN) = d'\Ln(20e\rN/d')$, and the decrease of $\Frac{\Ln(x)}{x}$ for $x > 0$.
    As noted above, this suffices to establish $E_{1}$.

    Thus, if we can show that each event $E_{1, 1}$, $E_{1, 2}$, and $E_{1, 3}$ happens with probability at least $1 - \delta/5$ over $\rsS_{1}, \rb$, it then follows from a union bound that they happen simultaneously with probability at least $1 - 3\delta/5$, as desired.

    For $E_{1, 1}$, since $\Ev{\abs{\rN}} = pn/2$, it follows by a multiplicative Chernoff bound that
    \begin{align}
        \Prob_{\rsS_{1} \sim \dD^{n/2}}*{\rN \geq \frac{pn}{4}}
        &\geq 1 - \exp(-pn/16)
        \\&\geq 1 - \delta/5 \tag{as we consider the case $p \geq 16\ln(\Frac{5}{\delta})/n$}
        .
    \end{align}

    For $E_{1, 2}$, by the law of total probability,
    \begin{align}
        \Prob_{\rsS_{1} \sim \dD^{n/2}}{E_{1, 2}}
        &\geq \sum_{N = 1}^{n/2} \Prob_{\rsS_{1} \sim \dD^{n/2}}{\rN = N} \Prob_{\rsS_{1} \sim \dD^{n/2}}*{E_{1, 2} \given \rN = N}
        \\&= \sum_{N = 1}^{n/2} \Prob_{\rsS_{1} \sim \dD^{n/2}}{\rN = N} \Prob_{\rsS_{\fs} \sim \dD_{\fs}^{N}}*{E_{1, 2} \given \rN = N} \tag{by the def.\ of $\rN$ and $\dD_{\fs}$}
        \\&\geq \sum_{N = 1}^{n/2} \Prob_{\rsS_{1} \sim \dD^{n/2}}{\rN = N} \Br*{1 - \Frac*{\delta}{5}} \tag{by \cref{lem:maurerandpontil}}
        \\&= 1 - \Frac{\delta}{5}
        .
    \end{align}

    The result for $E_{1, 3}$ is less immediate.
    Let $S_{1}$ be a realization of $\rsS_{1}$,
    and let $\cA_\mathrm{ada}$ denote \methodname{AdaBoost} with equal weighting of the hypotheses (\cref{alg:adaboost}).
    We will show that with probability at least $1 - \delta/5$ over $\rb$, we can, for $R$ rounds, simulate running $\cA_\mathrm{ada}$ on the training sequence $\rsS_{\fs}$, using hypotheses from $\rfB = \cup_{r = 1}^{R} \rfB_{r}$.
    This implies, by Lemmas~\ref{lem:adaboostconvergence} and \ref{lem:dualVC-dimensionofcombination}, that $\sign(\rfB^{(T)})$ contains a hypothesis with zero loss on $\rsS_{\fs}$, as desired.

    To show this, note that as $\fs \in \fF$, for any $\dD' \in \DistribOver(\rsS_{\fs})$ it holds that $\sup_{f \in \fF} \cor_{\dD'}(f) = 1$.
    So, from the weak learner guarantee,
    \begin{align}
        1 - \delta_{0}
        &\leq \Prob_{\rsS' \sim \dD'^{m_{0}}, \rb' \sim \dU}{\cor_{\dD'}(\cW(\rsS', \rb')) \geq \gamma_{0} \sup_{f \in \fF} \cor_{\dD'}(f) - \eps_{0}}
        \\&= \Prob_{\rsS', \rb'}{\cor_{\dD'}(\cW(\rsS', \rb')) \geq \gamma_{0} - \eps_{0}}
        \\&= \Prob_{\rsS', \rb'}{\ls_{\dD'}(\cW(\rsS', \rb')) \leq 1/2 - (\gamma_{0} -\eps_{0})/2}
        ,
    \end{align}
    where the last equality holds since $\ls_{\dD'}(f) = (1 - \cor_{\dD'}(f))/2$ for any $f\colon \InputSpace \to \{\pm 1\}$.
    Thus, in the case where $\dD'$ is realizable by $\fs$ we recover the realizable boosting guarantee with advantage $\theta \coloneqq (\gamma_{0}-\eps_{0})/2$.
    Moreover, by independence of $\rsS'$ and $\rb'$,
    \begin{align}
        \Prob_{\rsS' \sim \dD'^{m_{0}}, \rb' \sim \dU}{\ls_{\dD'}(\cW(\rsS', \rb')) \leq 1/2 - \theta}
        &= \sum_{\sS' \in \rsS_{\fs}^{m_{0}}} \Prob_{\rb'}{\ls_{\dD'}(\cW(\sS', \rb')) \leq 1/2 - \theta} \Prob_{\rsS'}{\rsS' = \sS'}
        .
    \end{align}
    Applying this to the previous inequality, we obtain that for any $\dD'$ over $\rsS_{\fs}$ there exists a deterministic training sequence $\sS'$ in the set of all possible training sequence of size $m_{0}$ over $\rsS_{\fs}$, such that
    \begin{align}
        \Prob_{\rb' \sim \dU}{\ls_{\dD'}(\cW(\sS', \rb')) \leq 1/2 - \theta}
        \geq 1 - \delta_{0}
        .
    \end{align}

    We will use this observation to simulate the execution of $\cA_\mathrm{ada}$ on $\rsS_{\fs}$ using hypotheses from $\rfB$.
    We do so by, at each round $r \in [R]$, at \cref{line:adaboost:get-weak-hypothesis}, letting $\cA_\mathrm{ada}$ receive $\rh \in \argmin_{h \in \rfB_{r}}\ls_{\rD_{r}}(h)$, where $\rD_{r}\in \DistribOver(\rsS_{\fs})$ is the distribution obtained at the end of round $r-1$.
    For $i \in [R]$, let
    \begin{gather}
        E'_{i} = \Set*{\min_{\rh \in \rfB_{i}}\ls_{\rD_{i}}(\rh) \leq 1/2 - \theta}
        \shortintertext{and}
        E' = \bigcap_{i = 1}^{R}E_{i}'
    \end{gather}
    be events over $\rsS_{1}$ and $\rb$.
    That is, $E'$ ensures that $\rfB$ is sufficient to simulate $\cA_\mathrm{ada}$ using weak hypotheses with advantage $\theta$.
    By \cref{lem:adaboostconvergence} with $\theta$ (which is valid since $\theta = (\gamma_{0} - \eps_{0})/2 < 1/2$) and $R = \Ceil{\ln(n)/\theta^{2}} \geq \Ceil{\ln(2\abs{\rsS_{\fs}})/\theta^{2}}$, the resulting output $\rv_\mathrm{ada}$ of $\cA_\mathrm{ada}$ belongs to $\rfB^{(R)}$ and satisfies that $\ls_{\rsS_{\fs}}^{\theta/2}(v_\mathrm{ada}) = 0$.
    Thus, by \cref{lem:dualVC-dimensionofcombination}, there exists $\rv \in \rfB^{(T)}$ with $T = \Ceil{\min\{\ln(n)/\theta^{2}, 260^2(4d^{*} + 2)/\theta^{2}\}}$ such that $\ls_{\rsS_{\fs}}(\rv) = 0$, which implies $E_{1, 3}$.
    Hence, to show $E_{1, 3}$ happens with probability at least $1 - \delta/5$, it suffices to show that $E'$ happens with this probability.

    We have that
    \begin{align}
        \Prob_{\rb}{E'}
        = \prod_{i = 1}^{R}\Prob_{\rb}[\bbbig]{E'_{i} \given \bigcap_{j = 1}^{i-1} E'_{j}}
        .
    \end{align}
    We now show that for each $i \in [R]$,
    \begin{align}
        \Prob_{\rb}[\bbbig]{E'_{i} \given \bigcap_{j = 1}^{i-1} E'_{j}}
        \geq 1 - \delta/(5R)
        ,
        \label{eq:conditional}
    \end{align}
    whereby we get that $\Prob_{\rb}{E'} \geq (1 - \delta/(5R))^{R}$, which, by Bernoulli's inequality, is at least $1 - \delta/5$.

    Fix $i \in [R]$.
    To show \cref{eq:conditional}, we notice that $\bigcap_{j = 1}^{i-1} E_{j}'$ depends only on the random distributions $\rD_{1}, \ldots, \rD_{i-1}$, which are a function of the random variables $\rb_{r, I, m}$ for $r \in [i-1], I \in [n/2]^{m_{0}}, m \in [M]$.
    Furthermore, these are independent of the $\rb_{i, I, m}$'s for any $I \in [n/2]^{m_{0}}$ and $m \in [M]$ used to construct $\rfB_{i}$.
    Moreover, by \cref{line:adaboost:newD}, $\rD_{i}$ also depends only on $\rb_{r, I, m}$ for $r \in [i-1], I \in [n/2]^{m_{0}}, m \in [M]$.
    Thus, considering realizations $b_{r, I, m}$ for $r \in [i-1], I \in [n/2]^m_{0}, m \in [M]$,
    under $\bigcap_{j = 1}^{i-1} E_{j}'$ we have that $D_{i}$ is fixed (and supported on $\rsS_{\fs}$ as the initial distribution of \cref{alg:adaboost} is uniform over $\rsS_{\fs}$ and the update rule in Lines~\ref{line:adaboost:Z} to \ref{line:adaboost:newD} preserves this), thus by the above argument there exists a training sequence $\sS' \in \rsS_{\fs}^{m_{0}}$ such that
    \begin{align}
        \Prob_{\rb}*{\ls_{\dD'}(\cW(S', \rb)) \leq 1/2 - \theta}
        &\geq 1 - \delta_{0}
        ;
    \end{align}
    and for some $I' \in [n/2]^{m_{0}}$ it holds that $\rsS_{1}\restriction{I'} = S'$.
    Therefore, as $\{S', \rb_{i, I', m}\}_{m \in [M]}$ are independent, with probability at least $1 - \delta_{0}^{M}$ it holds that $\Set{\cW(S', \rb_{i, I', m})}_{m \in [M]}$ contains a hypothesis $h$ such that $\ls_{\rD_{i}}(h) \leq 1/2 - \theta$, implying that $\Prob_{\rb}{E'_{i} \given \bigcap_{j = 1}^{i-1} E_{j}'} \geq 1 - \delta_{0}^{M}$, which is at least $1 - \delta/(5R)$ since $M = \Ceil{\ln(5R/\delta)/\ln(1/\delta_{0})}$, concluding the proof of the claim that $E_{1}$ occurs with probability at least $1 - 3\delta/5$.

\paragraph{Showing that $E_{2}$ happens with high probability:}
    Recall that we let
    \begin{align}
        E_{2}
        = \Set*{\exists v_\mathrm{g} \in \argmin_{v \in \sign(\rfB^{(T)})}\ls_{\dD}(v) : \ls_{\rsS_{2}}(v_\mathrm{g}) \leq \ls_{\dD}(v_\mathrm{g}) + \sqrt{\frac{4\ls_{\dD}(v_\mathrm{g})\ln(\Frac{5}{\delta})}{3n}} + \frac{4\ln(\Frac{5}{\delta})}{n}}
    \end{align}
    be an event over $\rsS_{1}$, $\rsS_{2}$, and $\rb$.

    For any realizations $b$ of $\rb$ and $S_{1}$ of $\rsS_{1}$, the set of hypotheses $\fB$ is fixed.
    Thus, for any fixed $v_\mathrm{g} \in \argmin_{v \in \sign(\fB^{(T)})} \ls_{\dD}(v)$, it follows by $\rsS_{2} \sim \dD^{n/2}$ and \cref{lem:shaisbernstein} that with probability at least $1 - \delta/5$ over $\rsS_{2}$ it holds that
    \begin{align}
        \ls_{\rsS_{2}}(v_\mathrm{g})
        \leq \ls_{\dD}(v_\mathrm{g}) + \sqrt{\frac{4\ls_{\dD}(v_\mathrm{g})\ln(\Frac{5}{\delta})}{3n}} + \frac{4\ln(\Frac{5}{\delta})}{n}
        .
    \end{align}
    So, by independence of $\rsS_{1}$, $\rb$, and $\rsS_{2}$,
    \begin{align}
        \Prob_{\rsS_{1}, \rsS_{2}, \rb}{E_{2}}
        &= \Ev_{\rsS_{1}, \rb}*{\Prob_{\rsS_{2}}*{E_{2}}}
        \\&\geq 1 - \delta/5
        ,
    \end{align}
    as claimed.

\paragraph{Showing that $E_{3}$ happens with high probability:}
    Recall that we let
    \begin{align}
        E_{3}
        = \bigcap_{v \in \sign(\fH^{(T)})} E_{3, v}
    \end{align}
    with
    \begin{align}
        E_{3, v}
        = \Set*{\ls_{\dD}(v) \leq \ls_{\rsS_{2}}(v) + \sqrt{\frac{36\ls_{\rsS_{2}}(v)\Br[\big]{\Gamma(n) + \ln(\Frac{5}{\delta})}}{n}} + \frac{30\Br[\big]{\Gamma(n) + \ln(\Frac{5}{\delta})}}{n}}
    \end{align}
    be an event over $\rsS_{2}$.

    So, $E_{3}$ is a uniform convergence bound over the class $\sign(\fH^{(T)})$ on the sample $\rsS_2$.
    It follows directly from \cref{lem:maurerandpontil} that this event happens with probability at least $1 - \delta/5$.
\end{proof}

\section{Auxiliary results}\label{app:lowerbound}
For convenience, we provide the full statement of the lower bound of \citet{Us25} on the sample complexity of agnostic boosting.

\begin{theorem}[\protect{\citet[Theorem 1.2]{Us25}}]\label{thm:lowerbound}
  There exist universal positive constants $C_1$, $C_2$, $C_3$, and $C_{4}$ for which the following holds.
  Given any $L \in (0, 1/2)$,
  $\gamma_{0}, \eps_{0}, \delta_{0} \in (0, 1]$,
  and integer $d \ge C_1 \ln(1/\gamma_{0}^{2})$,
  for $m_{0} = \Ceil[\big]{C_2 d \ln\Br{\Frac[][\Br]{1}{\delta_{0}\gamma_{0}^{2}}} / \eps_{0}^{2}}$
  there exist
  domain $\InputSpace$,
  and classes $\fF, \fH \subseteq \Set{\pm 1}^{\InputSpace}$, with $\VCdim(\fH) \le d$,
  such that
  there exists a $(\gamma_{0}, \eps_{0}, \delta_{0}, m_{0})$ agnostic weak learner for $\fF$ with base class $\fH$
  and, yet, the following also holds.
  For any learning algorithm $\cA\colon (\InputSpace \times \Set{\pm 1})^* \to \Set{\pm 1}^\InputSpace$
  there exists a distribution $\dD$ over $\InputSpace \times \Set{\pm 1}$
  such that $\errst = \inf_{f \in \fF} \err_{\dD}(f) = L$
  and for sample size $m \geq C_{3}\frac{d}{\gamma_{0}^{2} L}\frac{1}{(1 - 2L)^{2}}$
  we have with probability at least $1/50$ over $\rsS \sim \dD^{m}$ that
  \begin{align}
    \err_{\dD}(\cA(\rsS))
    \ge \errst + \sqrt{C_4 \errst \cdot \frac{d}{\gamma_{0}^{2} m \ln(2/\gamma_{0})}}
    .
  \end{align}
  Furthermore,
  for all $\eps \in (0, \sqrt{2}/2]$,
  $\delta \in (0, 1)$,
  and any learning algorithm $\cA$,
  there exists a data distribution $\dD$
  such that if $m \le \ln(1/(4\delta))/(2\eps^{2})$,
  then with probability at least $\delta$ over $\rsS \sim \dD^{m}$
  we have that
  \begin{align}
    \err_{\dD}(\cA(\rsS))
    \ge \errst + \eps
    ,
  \end{align}
  and for any $\eps \in (0, \sqrt{2}/16]$,
  and any learning algorithm $\cA\colon (\InputSpace \times \Set{\pm 1})^{*} \to \Set{\pm 1}^\InputSpace$
  there exists a data distribution $\dD$ such that
  if $m < \frac{d}{2048\gamma_{0}^{2}\log_{2}(2/\gamma_{0}^{2})\eps^{2}}$,
  then with probability at least $1/8$ over $\rsS \sim \dD^{m}$ we have that
  \begin{align}
    \err_{\dD}(\cA(\rsS))
    \ge \errst + \eps
    .
  \end{align}
\end{theorem}

\section{Deferred proofs}\label{app:auxiliary}
In this section, we provide the proofs of the lemmas used to prove \cref{thm:main}.

\begin{proof}[of \cref{lem:adaboostconvergence}]
    Let $\sS = \Br[\big]{(x_{1}, y_{1}), \ldots, (x_{n}, y_{n})}$ be the input data to \cref{alg:adaboost}.

    First, we note that
    \begin{align}
        \ls_{S}^{\theta/2}\Br*{\frac{1}{R}\sum_{r = 1}^{R} h_{r}}
        &= \Prob_{(\rx, \ry) \sim D_{1}}*{\frac{1}{R}\sum_{r = 1}^{R} h_{r}(\rx)\ry \leq \theta/2}
        \\&= \Prob_{(\rx, \ry) \sim D_{1}}*{\sum_{r = 1}^{R} \alpha \ry h_{r}(\rx) \leq \theta R \alpha/2}
        \\&\le \Ev_{(\rx, \ry) \sim D_{1}}*{\exp\Br*{\theta R \alpha/2 - \sum_{r = 1}^{R} \alpha \ry h_{r}(\rx)}}
        \\&= \exp\Br*{\theta R \alpha/2} \Ev_{(\rx, \ry) \sim D_{1}}*{\exp\Br*{-\sum_{r = 1}^{R} \alpha \ry h_{r}(\rx)}}
        ,
    \end{align}
    where the first equality uses that $D_{1}$ is the uniform distribution over $S$ (by \cref{line:adaboost:initD}), and the inequality follows from Markov's inequality, i.e., $\indicator{a \leq b} = \indicator{0 \leq b-a} \leq \exp(b-a)$.
    The expectation term can be rewritten as:
    \begin{align}
        \Ev_{(\rx, \ry) \sim D_{1}}*{\exp\Br*{-\sum_{r = 1}^{R} \alpha h_{r}(\rx)\ry}}
        &= \sum_{i = 1}^{n} D_{1, i} \exp\Br*{-\sum_{r = 1}^{R} \alpha h_{r}(x_{i})y_{i}}
        \\&= Z_{1}\sum_{i = 1}^{n} D_{2, i} \exp\Br*{-\sum_{r = 2}^{R} \alpha h_{r}(x_{i})y_{i}}
        \\&\mathbin{\phantom{=}} \enspace\vdots
        \\&= \prod_{r = 1}^{R} Z_{r}
        ,
    \end{align}
    where the first equality is by the definition of expectation, the second equality is by the definition of $D_{r + 1, i} = D_{r, i}\exp(-\alpha y_{i} h_{r}(x_{i}))/Z_{r}$ in \cref{line:adaboost:newD} and the last equality is by iterating the argument and $D_{R + 1}$ being a probability distribution.
    Thus, we have that
    \begin{align}
        \ls_{S}^{\theta/2}\Br*{\frac{1}{R}\sum_{r = 1}^{R} h_{r}}
        &\leq \exp\Br*{\theta R \alpha/2} \prod_{r = 1}^{R} Z_{r}
        .
    \end{align}

    We will show that $\prod_{r = 1}^{R} Z_{r} \leq \exp(R\ln(2\sqrt{(1/2 - \theta)(1/2 + \theta)}))$, which combined with $\alpha = \frac{1}{2} \ln((1/2 + \theta)/(1/2 - \theta)) = \ln(\sqrt{(1/2 + \theta)/(1/2 - \theta)})$ implies that
    \begin{align}
        \ls_{S}^{\theta/2}\Br*{\frac{1}{R}\sum_{r = 1}^{R} h_{r}}
        &\leq \exp\Br*{R \ln\Br*{\Br*{\frac{1/2 + \theta}{1/2 - \theta}}^{\theta/4}} + R\ln\Br*{2\sqrt{(1/2 - \theta)(1/2 + \theta)}}}
        \\&= \exp\Br*{R \ln\Br[\big]{2(1/2 + \theta)^{1/2 + \theta/4}(1/2 - \theta)^{1/2 - \theta/4}}}
        .
    \end{align}
    We will furthermore show that $\ln(2(1/2 + \theta)^{1/2 + \theta/4}(1/2 - \theta)^{1/2 - \theta/4}) \leq \ln(1 - \theta^{2}) \leq -\theta^{2}$ (where the last inequality is by $\ln(1 + x) \leq x$ for $x > -1$), which by the above and $R = \lceil\ln(2n)/\theta^{2}\rceil$ we get that $\ls_{S}^{\theta/2}(\frac{1}{R}\sum_{r = 1}^{R} h_{r}) < \frac{1}{2n}$, which yields $\ls_{S}^{\theta/2}(\frac{1}{R}\sum_{r = 1}^{R} h_{r}) = 0$ and would conclude the proof.
    We now show the two remaining claims.
    To this end, we note that
    \begin{align}
        Z_{r}
        &= \sum_{j = 1}^{n} D_{r, j} \exp\Br*{-\alpha y_{j}h_{r}(x_{j})}
        \\&= \sum_{h_{r}(x_{j}) = y_{j}} D_{r, j} \exp\Br*{-\alpha} + \sum_{h_{r}(x_{j})\neq y_{j}} D_{r, j} \exp\Br*{\alpha}
        \\&= \exp(-\alpha) + (\exp(\alpha)-\exp(-\alpha))\sum_{h_{r}(x_{j})\neq y_{j}} D_{r, j}
        \\&\le \exp(-\alpha) + (\exp(\alpha)-\exp(-\alpha))(1/2 - \theta)
        ,
    \end{align}
    where in the inequality we have used that $\alpha > 0$, so $\exp(\alpha) - \exp(-\alpha) \ge 0$, and that the error of $h_{r}$ under $D_{r}$ is at most $1/2 - \theta$.
    Furthermore, since $\alpha = \ln(\sqrt{(1/2 + \theta)/(1/2 - \theta)})$, we have the following calculation:
    \begin{align}
        &\exp(-\alpha) + (\exp(\alpha)-\exp(-\alpha))(1/2 - \theta)
        \\&\qquad= \sqrt{\frac{1/2 - \theta}{1/2 + \theta}} + \Br*{\sqrt{\frac{1/2 + \theta}{1/2 - \theta}}-\sqrt{\frac{1/2 - \theta}{1/2 + \theta}}} \Br*{1/2 - \theta}
        \\&\qquad= \sqrt{(1/2 - \theta)(1/2 + \theta)}\Br*{\frac{1}{1/2 + \theta} + \Br*{\sqrt{\frac{1/2 + \theta}{1/2 - \theta}}-\sqrt{\frac{1/2 - \theta}{1/2 + \theta}}} \sqrt{\frac{1/2 - \theta}{1/2 + \theta}}}
        \\&\qquad= \sqrt{(1/2 - \theta)(1/2 + \theta)}\Br*{\frac{1}{1/2 + \theta} + \Br*{1 - \frac{1/2 - \theta}{1/2 + \theta}}}
        \\&\qquad= \sqrt{(1/2 - \theta)(1/2 + \theta)}\Br*{\frac{1}{1/2 + \theta} + \frac{2\theta}{1/2 + \theta}}
        \\&\qquad= 2\sqrt{(1/2 - \theta)(1/2 + \theta)}
        ,
    \end{align}
    which implies that $Z_{r} \le \exp(\ln(2\sqrt{(1/2 - \theta)(1/2 + \theta)}))$,
    and that
    \begin{align}
        \prod_{r = 1}^{R} Z_{r}
        \le \exp\Br*{R\ln\Br*{2\sqrt{(1/2 - \theta)(1/2 + \theta)}}}
        ,
    \end{align}
    as claimed.
    For the claim $\ln(2(1/2 + \theta)^{1/2 + \theta/4}(1/2 - \theta)^{1/2 - \theta/4}) \leq \ln(1 - \theta^{2})$ we show that for $x \in [0, 1/2)$ it holds that $f(x) = \ln(2(1/2 + x)^{1/2 + x/4}(1/2 - x)^{1/2 - x/4}) - \ln(1 - x^{2})$ is non-positive.
    To this end, we note that the first derivative $f'(x)$ is
    \begin{align}
        &\frac{\ln\Br*{\frac{1}{2} + x} - \ln\Br*{\frac{1}{2} - x} - 5x^2\ln\Br*{\frac{1}{2} + x} + 5x^{2}\ln\Br*{\frac{1}{2} - x}}{16x^{4} - 20x^{2} + 4}
        \\&\qquad+ \frac{4x^{4}\ln\Br*{\frac{1}{2} + x} - 4x^{ 4}\ln\Br*{\frac{1}{2} - x} - 20x^{3} - 4x}{16x^{4} - 20x^{2} + 4}
        ,
    \end{align}
    and the second derivative is
    \begin{align}
        f''(x)
        = \frac{16x^{6} + 46x^{4}-26x^{2}}{16x^{8}-40x^{6} + 33x^{4}-10x^{2} + 1}
        .
    \end{align}
    We notice that for $x \in [0, 1/2)$ we have that $f''(x) \leq 0$, implying that $f'(x)$ is non-increasing.
    Since $f'(0) = 0$, we get that $f'(x) \leq 0$ for all $x \in [0, 1/2)$, which implies that $f(x)$ is non-increasing on $[0, 1/2)$ and thus $f(x) \leq f(0) = 0$ for all $x \in [0, 1/2)$, as claimed.
\end{proof}

\begin{proof}[of \cref{lem:RcombinationVC}]
    Let $T$ be any natural number larger than $1$.
    Let $\{x_{1}, \ldots, x_{d'}\}$ be any point set consisting of $d' \geq d$ points, where $d = \VCdim(\fH)$.
    By the Sauer-Shelah lemma, the class $\fH$ can produce at most $(ed'/d)^{d}$ different labelings of $\{x_{1}, \ldots, x_{d'}\}$.
    This implies that $\fH^{(T)}$ can produce at most $(ed'/d)^{dT}$ different outputs on $\{x_{1}, \ldots, x_{d'}\}$.
    This, in turn, means that $\sign(\fH^{(T)})$ can produce at most $(ed'/d)^{dT}$ different labelings of $\{x_{1}, \ldots, x_{d'}\}$ (we recall that we set $\sign(0) = 1$).
    Now, let $\{x_{1}, \ldots, x_{d'}\}$ be a point set of maximal size shattered by $\sign(\fH^{(T)})$.
    If $d' \leq d$, we are done, so assume $d' > d$.
    By the above observations, $2^{d'}$, the number of different labelings of $\{x_{1}, \ldots, x_{d'}\}$, must be at most $(ed'/d)^{dT}$.
    Taking the natural logarithm of both sides,
    \begin{align}
        d'\ln(2) \leq dT\ln(ed'/d)
        \iff \frac{\ln(2)}{eT} \leq \frac{d\ln(ed'/d)}{ed'}
        .
        \label{eq:rcombinationvc1}
    \end{align}
    We will now argue that this can only hold if $d' \leq 4Td\ln(4eT)$, which will conclude the proof.
    To this end we remark that $\ln(x)/x$ is decreasing for $x \geq e$ (has derivative $(1 - \ln(x))/x^{2}$).
    Using $x = ed'/d > e$ by the assumption $d' > d$, we get that for $x > 4eT\ln(4eT)$,
    \begin{align}
        \frac{\ln(x)}{x}
        &\leq \frac{\ln\Br*{4eT\ln(4eT)}}{4eT\ln(4eT)}
        \\&< \frac{\ln\Br*{(4eT)^{2}}}{4eT\ln(4eT)}
        \\&= \frac{1}{2eT}
        \\&< \frac{\ln(2)}{eT}
        ,
    \end{align}
    where we have used in the first inequality that $\ln(x)/x$ is decreasing for $x \geq e$ and in the second inequality that $\ln(4eT) < 4eT$ and that $\ln$ is strictly increasing.
    Thus, we conclude from this and \cref{eq:rcombinationvc1} that it cannot be the case that $ed'/d > 4eT\ln(4eT)$ implying that $d' \leq 4Td\ln(4eT)$ as claimed.
\end{proof}

In the proof of \cref{lem:dualVC-dimensionofcombination} we will use the following uniform convergence result, which for instance can be found in \citet[Theorem~6.8]{UML}.

\begin{lemma}\label{lem:uniformconvergence}
    For any $0 < \delta < 1$ and any hypothesis class $\fH$ of VC dimension $d$ and any distribution $\dD$ over $\InputSpace \times \{\pm 1\}$, it holds that with probability at least $1 - \delta$ over the choice of a training sequence $\rsS \sim \dD^{n}$ it holds that for all $h \in \fH$
    \begin{align}
        |\ls_{\dD}(h)-\ls_{\rsS}(h)|
        &\leq 62\sqrt{\frac{2d + 1}{n}} + \sqrt{\frac{2\ln(1/\delta)}{n}}
        .
    \end{align}
\end{lemma}

\begin{proof}[of \cref{lem:dualVC-dimensionofcombination}]
    We recall that for a hypothesis class $\fH \subseteq \{\pm 1\}^\InputSpace$, the dual VC dimension is the largest integer $d^{*}$ such that there exist $h_{1}, \ldots, h_{d^{*}}$ for which, given any $y \in \{\pm 1\}^{d^{*}}$, there exists an $x \in \InputSpace$ such that $(h_{1}(x), \ldots, h_{d^{*}}(x)) = y$.
    We remark that if for each $x \in \InputSpace$ we define $p_{x}(h) = h(x)$ as a mapping from $\fH$ to $\{\pm 1\}$ and consider the class $\cP_{+} = \{p_{x}\}_{x \in \InputSpace}$, then the dual VC dimension of $\fH$ is equal to $\VCdim(\cP_{+})$.
    We define $\cP_{-} = -\cP_{+} = \{-p_{x}\}_{x \in \InputSpace}$.
    Since the VC dimension is invariant under negation of all hypotheses, and the VC dimension of the union of two classes is at most the sum of their VC dimensions plus one, we have that $\VCdim(\cP_{-}\cup\cP_{+}) \leq 2\VCdim(\cP_{+}) + 1$.
    Thus, if for $(x, y) \in \InputSpace \times \{\pm 1\}$ we let $p_{(x, y)}(h) = yh(x)$, then the class $\cP = \{p_{(x, y)} \mid (x, y) \in \InputSpace \times \{\pm 1\}\} = \cP_{-}\cup\cP_{+}$ has VC dimension at most $2d^{*} + 1$.

    By the assumptions in the lemma, $\{\alpha_{h}\}_{h \in \fH}$ is a probability distribution over $\fH$.
    We can extend this to a probability distribution $\dD_{\alpha}$ over $\fH \times \{\pm 1\}$ which, with probability $\alpha_{h}$, outputs $(h, 1)$.
    By \cref{lem:uniformconvergence} (with hypothesis class $\cP \subseteq \fH\to \{\pm 1\}$ and distribution $\dD_{\alpha}$ over $\fH \times \{\pm 1\}$), we know that for $L = \lceil 130^2(4d^{*} + 2)/\theta^{2}\rceil \geq 62^2(4d^{*} + 2)/(48\theta/100)^{2}$, a set of $L$ i.i.d.\ drawn hypotheses $(\rh_{r}, 1)$ from $\dD_{\alpha}$ satisfies with probability at least $1/2$ that for all $(x, y) \in \InputSpace \times \{\pm 1\}$:
    \begin{align}
        \left|\Ev_{(\rh, 1) \sim \dD_{\alpha}}*{\indicator{p_{(x, y)}(\rh) \ne 1}} - \frac{1}{L}\sum_{r = 1}^{L} \indicator{p_{(x, y)}(\rh_{r}) \ne 1}\right|
        &\leq \frac{48\theta}{100} \Br*{1 + \frac{\sqrt{\ln(2)}}{62}}
        \\&\leq \frac{49\theta}{100}
        .
    \end{align}
    As this event occurs with positive probability, there must exist realization $(h_{1}, 1), \ldots, (h_{L}, 1)$ for which the above holds.
    Let us consider such a realization.
    We note that each $h_{i}$ must have $\alpha_{h_{i}} > 0$.
    Furthermore, since $ \indicator{p_{(x, y)}(\rh) \ne 1} = \frac{1}{2}-\frac{1}{2}p_{(x, y)}(\rh)$, the inequality implies:
    \begin{align}
        \left|\Ev_{(\rh, 1) \sim \dD_{\alpha}}*{p_{(x, y)}(\rh)}- \frac{1}{L}\sum_{r = 1}^{L} p_{(x, y)}(\rh_{r})\right|
        \leq \frac{49\theta}{100}
        .
    \end{align}
    Finally, by definition, $\Ev_{(\rh, 1) \sim \dD_{\alpha}}*{p_{(x, y)}(\rh)} = \Ev_{(\rh, 1) \sim \dD_{\alpha}}*{\rh(x)y} = \sum_{h \in \fH}\alpha_{h}h(x)y = v(x)y$, and $\frac{1}{L}\sum_{r = 1}^{L} p_{(x, y)}(h_{r}) = \frac{1}{L}\sum_{r = 1}^{L} h_{r}(x)y$.
    Thus, in the case that $(x, y) \in S$ where we have that $v(x)y \geq \theta$, it follows that $\frac{1}{L}\sum_{r = 1}^{L} h_{r}(x)y \geq \theta /50$.
    Thus, if we let $\tilde{v} = \frac{1}{L}\sum_{r = 1}^{L} h_{r}$, it satisfies the claim of the lemma.
\end{proof}

To give the proof of \cref{lem:uniformconvergence}, we will use the following bound on the Rademacher complexity of a hypothesis class in terms of its VC dimension.

\begin{lemma}[\citet{AFoL_Lecture05}, Theorem~5.6]\label{lem:rademacherVC}
    Let $\fH \subseteq \Set{0, 1}^\InputSpace$ be a hypothesis class of VC dimension $d$.
    Then, for any set of $n$ points $x_{1}, \ldots, x_{n} \in \InputSpace$, it holds that
    \begin{align}
        \Ev_{\sigma \sim \{\pm 1\}^{n}}*{\sup_{h \in \fH}\frac{1}{n}\sum_{i = 1}^{n}\sigma_{i}h(x_{i})}
        \leq 31\sqrt{\frac{d}{n}}
        .
    \end{align}
\end{lemma}

\begin{proof}[of \cref{lem:uniformconvergence}]
    We have that
    \begin{align}
        \sup_{h \in \fH} \left|\ls_{\dD}(h)-\ls_{\rsS}(h)\right|
        &= \frac{1}{2}\sup_{h \in \fH} \left|\corr_{\dD}(h)-\corr_{\rsS}(h)\right|
        \\&= \frac{1}{2}\bbbbigl(\;\underbrace{\sup_{h \in \fH} \left|\corr_{\dD}(h)-\corr_{\rsS}(h)\right| - \Ev_{\rsS \sim \dD^{n}}[\bbbig]{\sup_{h \in \fH}\left|\corr_{\dD}(h)-\corr_{\rsS}(h)\right|}}_{\spadesuit}
        \\&\qquad\quad+ \underbrace{\Ev_{\rsS \sim \dD^{n}}[\bbbig]{\sup_{h \in \fH}\left|\corr_{\dD}(h)-\corr_{\rsS}(h)\right|}}_{\clubsuit}\;\bbbbigr)
    \end{align}
    where the first equality is by the relation $\ls(h) = \frac{1}{2}(1 - \corr(h))$, and the second equality is by adding and subtracting the expectation term.
    We will now bound the terms $\spadesuit$ and $\clubsuit$ with probability at least $1 - \delta$, by respectively $\sqrt{2\ln(1/\delta)/n}$ and $4 \cdot 31\sqrt{(2d + 1)/n}$.
    This implies the claim of the lemma with $C = 2(4 \cdot 31)^{2} + 2$.
    Since changing one example in $\rsS = ((\rx_{1}, y_{1}), \ldots, (\rx_{n}, \ry_{n}))$ can only change $\corr_{\rsS}(h)$ by $1/n$ in absolute value, it follows by McDiarmid's inequality with $\eps = \sqrt{2\ln(1/\delta)/n}$, that
    \begin{align}
        \Prob_{\rsS \sim \dD^{n}}*{\sup_{h \in \fH} \abs{\corr_{\dD}(h)-\corr_{\rsS}(h)} - \Ev_{\rsS \sim \dD^{n}}*{\sup_{h \in \fH}\abs{\corr_{\dD}(h)-\corr_{\rsS}(h)}} \geq \eps}
        &\leq \exp\Br*{\frac{-2\eps^{2}}{n(2/n)^{2}}}
        \\&\leq \delta
        .
    \end{align}
    Furthermore, by symmetrization and the Rademacher complexity definition, we have that
    \begin{align}
        \Ev_{\rsS \sim \dD^{n}}*{\sup_{h \in \fH}\abs{\corr_{\dD}(h)-\corr_{\rsS}(h)}}
        &\leq \Ev_{\rsS, \rsS' \sim \dD^{n}}*{\sup_{h \in \fH}\abs{\corr_{\rsS'}(h)-\corr_{\rsS}(h)}}
        \\&= \frac{1}{n} \Ev_{\sigma \sim \{\pm 1\}^{n}}*{\Ev_{\rsS, \rsS' \sim \dD^{n}}*{\sup_{h \in \fH}\abs[\bbig]{\sum_{i = 1}^{n}\sigma_{i}(h(\rx_{i})\ry_{i}-h(\rx'_{i})\ry'_{i})}}}
        \\&\leq \frac{2}{n}\Ev_{\rsS \sim \dD^{n}}*{\Ev_{\sigma \sim \{\pm 1\}^{n}}*{\sup_{h \in \fH}\abs[\bbig]{\sum_{i = 1}^{n}\sigma_{i}h(\rx_{i})\ry_{i}}}}
        \\&= \frac{2}{n} \Ev_{\rsS \sim \dD^{n}}*{\Ev_{\sigma \sim \{\pm 1\}^{n}}*{\sup_{h \in \fH}\abs[\bbig]{\sum_{i = 1}^{n}\sigma_{i}h(\rx_{i})}}}
        \\&\leq \frac{2}{n} \Ev_{\rsS \sim \dD^{n}}*{\Ev_{\sigma \sim \{\pm 1\}^{n}}*{\sup_{h \in \fH\cup (-\fH)}\sum_{i = 1}^{n}\sigma_{i}h(\rx_{i})}}
        ,
    \end{align}
    where the first equality is by convexity of absolute value and taking supremum inside the expectation increase the expected value, the first equality is by symmetrization, the second inequality is by triangle inequality of absolute value, the second equality is by $\sigma_{i}\ry_{i}$ having the same distribution as $\sigma_{i}$, and the last inequality is by the supremum only becoming larger when taking the supremum over a larger class, and that, letting $\fG \coloneqq \fH \cup (-\fH)$, we have that $\sup_{h \in \fG}|\sum_{i = 1}^{n}\sigma_{i}h(\rx_{i})| = \max\{\sup_{h \in \fG}\sum_{i = 1}^{n}\sigma_{i}h(\rx_{i}), -\sup_{h \in \fG}\sum_{i = 1}^{n}\sigma_{i}h(\rx_{i})\} = \sup_{h \in \fG}\sum_{i = 1}^{n}\sigma_{i}h(\rx_{i})$.
    We furthermore have for any realization $S$ of $\rsS$ that
    \begin{align}
        \Ev_{\sigma \sim \{\pm 1\}^{n}}*{\sup_{h \in \fH\cup (-\fH)}\frac{1}{n}\sum_{i = 1}^{n}\sigma_{i}h(x_{i})}
        &= 2\Ev_{\sigma \sim \{\pm 1\}^{n}}*{\sup_{h \in \fH\cup (-\fH)}\frac{1}{n}\sum_{i = 1}^{n}\sigma_{i}\frac{(h(x_{i}) + 1)}{2}}
        \\&\leq 2 \cdot 31\sqrt{\frac{\VCdim(\fH\cup (-\fH))}{n}}
        \\&\leq 2 \cdot 31\sqrt{\frac{2d + 1}{n}}
        ,
    \end{align}
    where we in the first equality have used that $h = 2((h + 1 - 1)/2)$ and that the $\sigma_{i}$s has zero expectation, and the equality follows from \cref{lem:rademacherVC}, and the last inequality follows from $\VCdim(\fH\cup (-\fH))$ being at most $\VCdim(\fH) + \VCdim(-\fH) + 1 = 2d + 1$, which gives that
    \begin{align}
        \Ev_{\rsS \sim \dD^{n}}*{\sup_{h \in \fH}\left|\corr_{\dD}(h)-\corr_{\rsS}(h)\right|}
        &\leq 4 \cdot 31\sqrt{\frac{2d + 1}{n}}
        .
    \end{align}
\end{proof}

\end{document}